\documentclass[11pt]{article}

\usepackage[ulibertine=true]{uchitra_1}

\usepackage{microtype}
\usepackage{graphicx}
\usepackage{subfigure}
\usepackage{booktabs} 

\usepackage{hyperref}

\usepackage[numbers, sort, comma, square]{natbib}
\usepackage{amsthm}
\usepackage{amsmath}
\usepackage{amsfonts}

\usepackage[title]{appendix}

\usepackage[dvipsnames]{xcolor}

\usepackage{tikz} 
\usetikzlibrary{shapes,positioning,fit,calc}

\tikzstyle{vertex} = [fill,shape=circle,node distance=80pt]
\tikzstyle{edge} = [fill,opacity=.5,fill opacity=.5,line cap=round, line join=round, line width=50pt]
\tikzstyle{elabel} =  [fill,shape=circle,node distance=30pt]
\pgfdeclarelayer{background}
\pgfsetlayers{background,main}

\theoremstyle{plain}
\newtheorem{innercustomgeneric}{\customgenericname}
\providecommand{\customgenericname}{}
\newcommand{\newcustomtheorem}[2]{%
  \newenvironment{#1}[1]
  {%
   \renewcommand\customgenericname{#2}%
   \renewcommand\theinnercustomgeneric{##1}%
   \innercustomgeneric
  }
  {\endinnercustomgeneric}
}

\newcustomtheorem{customthm}{Theorem}

\DeclareMathOperator{\Var}{Var}

\begin{document}

\title{\vspace{-4em}Random Walks on Hypergraphs with \\ Edge-Dependent Vertex Weights \vspace{-0.5em}}
\author{Uthsav Chitra\thanks{Email: \href{uchitra@cs.princeton.edu}{uchitra@cs.princeton.edu}} , Benjamin J Raphael\thanks{Email: \href{braphael@cs.princeton.edu}{braphael@cs.princeton.edu} } }
\date{Department of Computer Science, Princeton University \\ \vspace{1em} \today}
\maketitle

\begin{abstract}
Hypergraphs are used in machine learning to model higher-order relationships in data. 
While spectral methods for graphs are well-established, spectral theory for hypergraphs remains an active area of research. 
In this paper, we use random walks to develop a spectral theory for \emph{hypergraphs with edge-dependent vertex weights}: hypergraphs where every vertex $v$ has a weight $\gamma_e(v)$ for each incident hyperedge $e$ that describes the contribution of $v$ to the hyperedge $e$. We derive a random walk-based hypergraph Laplacian, and bound the mixing time of random walks on such hypergraphs. Moreover, we give conditions under which random walks on such hypergraphs are equivalent to random walks on graphs. As a corollary, we show that current machine learning methods that rely on Laplacians derived from random walks on hypergraphs with edge-\emph{independent} vertex weights do not utilize higher-order relationships in the data. Finally, we demonstrate the advantages of hypergraphs with edge-dependent vertex weights on ranking applications using real-world datasets.



\end{abstract}

\section{Introduction}
\label{introduction}

Graphs are ubiquitous in machine learning,
where they are used to  represent pairwise relationships between objects. 
For example, social networks, protein-protein interaction (PPI) networks, and the internet are modeled with graphs.
One limitation of graph models, however, is that they do not encode higher-order relationships between objects. A social network can represent a community of users (e.g. a friend group) as a collection of edges between each user, but this pairwise representation loses information about the overall group structure \cite{hg_sn}. In biology, protein interactions are not only between pairs of proteins, but also between groups of proteins in protein complexes \cite{hg_complex,ritz}.

Such higher-order interactions can be modeled using a hypergraph: a generalization of a graph containing hyperedges that can be incident to more than two nodes.
A hypergraph representation of a social network can model a community of friends with a single hyperedge.  In contrast, the corresponding representation of a community in a graph requires many edges that connect pairs of individuals within the community; conversely, it may not be clear which collection of edges in a graph represents a community (e.g. a clique, an edge-dense subnetwork, etc).
Hypergraphs have been used in a variety of machine learning tasks, including clustering \cite{agarwal_beyond, zhou, panli_nips, panli_icml}, 
ranking keywords in a collection of documents \cite{bellaachia}, predicting customer behavior in e-commerce \cite{etail}, object classification \cite{gao_dynamic_hg,gao_object_classification}, and image segmentation \cite{hype_image_seg}.  

A common approach to incorporate  
graph information in a machine learning algorithm is to utilize properties of random walks or diffusion processes on the graph.  For example, random walks on graphs underlie algorithms for recommendation systems \cite{rw_rec}, clustering \cite{rw_clustering, rw_clustering_2}, information retrieval \cite{rw_pagerank}, and other applications. In many machine learning applications, the graph is represented through the graph Laplacian. Spectral theory includes many key results regarding the eigenvalues and eigenvectors of the graph Laplacian, and these results form the foundation of spectral learning algorithms.

Spectral theory on hypergraphs is much less developed than on graphs.  In seminal work, \citet{zhou} developed learning algorithms on hypergraphs based on random walks on graphs.  However, at nearly the same time, \citet{agarwal_higher_order} showed that the hypergraph Laplacian matrix used by Zhou et al. is equal to the Laplacian matrix of a closely related graph, the star graph.
A consequence of this equivalence is
that the methods introduced by Zhou et al. utilize only pairwise relationships between objects, rather than the higher-order relationships encoded in the hypergraph. 
More recently, \citet{louis_hg} and \citet{panli_nips, panli_icml} developed \emph{nonlinear} Laplacian operators for hypergraphs that partially address this issue. However, all existing constructions of linear Laplacian operators utilize only pairwise relationships between vertices, as shown by  \citet{agarwal_higher_order}.



In this paper, we develop a spectral theory for hypergraphs with \emph{edge-dependent vertex weights}.  In such a hypergraph, each hyperedge $e$ has an edge weight $\omega(e)$, and each vertex $v$ has a collection of vertex weights, with one weight $\gamma_e(v)$ for each hyperedge $e$ incident to $v$.  The edge-dependent vertex weight $\gamma_e(v)$ models the contribution of vertex $v$ to hyperedge $e$. 
Edge-dependent vertex weights have previously been used in several applications including: image segmentation, where the weights represent the probability of an image pixel (vertex) belonging to a segment (hyperedge) \cite{Ding2010}; e-commerce, where the weights model the quantity of a product (hyperedge) in a user's shopping basket (vertex) \cite{etail}; and text ranking, where the weights represent the importance of a keyword (vertex) to a document (hyperedge) \cite{bellaachia}.
Hypergraphs with edge-dependent vertex weights have also been used in image search \cite{Zeng2016, Huang2010} and 3D object classification \cite{gao_dynamic_hg}, where the weights represent contributions of vertices in a k-nearest-neighbors hypergraph.

Unfortunately, 
because of a lack of a spectral theory for
hypergraphs with edge-dependent vertex weights, many of the papers that use these hypergraphs rely on incorrect or theoretically unsound assumptions. For example,
\citet{gao_dynamic_hg} and \citet{Ding2010} use a hypergraph Laplacian with no spectral guarantees, while \citet{etail} derive an incorrect stationary distribution for a random walk on such a hypergraph (see Supplement for additional details).
The reason such issues arise is because existing spectral methods are developed for hypergraphs with \emph{edge-independent vertex weights}, i.e. hypergraphs where the $\gamma_e(v)$ are identical for all hyperedges $e$.

In this paper, we derive several results for hypergraphs with edge-dependent vertex weights.  First, we show that random walks on hypergraphs with edge-independent vertex weights are \emph{always} equivalent to random walks on the clique graph (Figure \ref{hg_ex2}).
This generalizes the results of \citet{agarwal_higher_order} and gives the underlying reason why existing constructions of hypergraph Laplacian matrices \cite{rodriguez, zhou}
 do not utilize the higher-order relations of the hypergraph.

Motivated by this result, we derive a random walk-based 
Laplacian matrix for hypergraphs with edge-dependent vertex weights that utilizes the
higher-order relations expressed in the hypergraph structure.
This Laplacian matrix satisfies the typical properties one would expect of a Laplacian matrix, including being positive semi-definite and satisfying a Cheeger inequality.
We also derive a formula for the stationary distribution of a random walk on a hypergraph with edge-dependent vertex weights, and give a bound on the mixing time of the random walk.

Our paper is organized as follows. In Section \ref{notation}, we define our notation, and introduce hypergraphs with edge-dependent vertex weights. In Section \ref{compare}, we formally define random walks on hypergraphs with edge-dependent vertex weights, and show that when the vertex weights are edge-independent, a random walk on a hypergraph has the same transition matrix as a random walk on its clique graph. In Section \ref{stat_dist_mix}, we derive a formula for the stationary distribution of a random walk, and use it to bound the mixing time. In Section \ref{laplac}, we derive a random-walk based Laplacian matrix for hypergraphs with edge-dependent vertex weights and show some basic properties of the matrix. Finally, in Section \ref{exp}, we demonstrate 
two applications of hypergraphs with edge-dependent vertex weights: ranking authors in a citation network and ranking players in a video game. All proofs are in the Supplementary Material.

\section{Graphs, Hypergraphs, and Random Walks}
\label{notation}

Let $G=(V,E,w)$ be a graph with vertex set $V$, edge set $E$, and edge weights $w$.
For a vertex $v$, let $N(v) = \{u \in V: (u,v) \in E\}$ denote the vertices incident to $v$. The \textit{adjacency matrix} $A$ of a graph is a $|V| \times |V|$ matrix where $A(u,v) = w(e)$ if $(u,v) \in E$ and $0$ otherwise.

Let $H = (V,E, \omega)$ be a \emph{hypergraph} with vertex set $V$; edge set $E \subset 2^V$; and hyperedge weights $\omega$. A graph is a special case of a hypergraph, where each hyperedge $e$ has size $|e| = 2$. For hypergraphs, the terms ``hyperedge" and ``edge" are used interchangeably. A random walk on a hypergraph is typically defined as follows \cite{zhou,hyper_rw_example,cover_times_hypergraph1,cover_times_hypergraph2}.
At time $t$,  a ``random walker" at vertex $v_t$ will:
\begin{enumerate}
\item
Select an edge $e$ containing $v_t$, with probability proportional to $\omega(e)$.
\item
Select a vertex $v$ from $e$, uniformly at random.
\item
Move to vertex $v_{t+1} = v$ at time $t+1$.
\end{enumerate}

A natural extension is to modify Step 2: instead of choosing $v$ uniformly at random from $e$, we pick $v$ according to a fixed probability distribution on the vertices in $e$. This motivates the following definition of a hypergraph with \emph{edge-dependent vertex weights}.

\begin{defn}
\label{hg_def}
A hypergraph $H=(V,E, \omega, \gamma)$ with \emph{edge-dependent vertex weights} is a set of vertices $V$, a set $E \subset 2^V$ of hyperedges, a weight $\omega(e)$ for every hyperedge $e \in E$, and a weight 
$\gamma_e(v)$ for every hyperedge $e \in E$ and every vertex $v$ incident to $e$.
\end{defn}
{We emphasize that
a vertex $v$  in a hypergraph with edge-dependent vertex weights has \emph{multiple} weights: one weight $\gamma_e(v)$ for each hyperedge $e$ that contains $v$.  Intuitively, $\gamma_e(v)$ measures the contribution of vertex $v$ to hyperedge $e$. In a random walk on a hypergraph with edge-dependent vertex weights, the random walker will pick a vertex $v$ from hyperedge $e$ with probability proportional to $\gamma_e(v)$.
Note that we set $\gamma_e(u) = 0$ if $u \not\in e$. We show an example of a hypergraph with edge-dependent vertex weights in Figure \ref{hg_ex2}.

If each vertex has the same contribution to all incident hyperedges, i.e. $\gamma_e(v) = \gamma_{e'}(v)$ for all hyperedges $e$ and $e'$ incident to $v$, then we say that the hypergraph has \emph{edge-independent vertex weights}, and we use $\gamma(v) = \gamma_e(v)$ to refer to the vertex weights of $H$. If $\gamma_e(v) = 1$ for all vertices $v$ and incident hyperedges $e$, we say the vertex weights are \emph{trivial}.

We define $E(v) = \{e \in E: v \in e\}$ to be the hyperedges incident to a vertex $v$, and $E(u,v) = \{e \in E : u \in e, v \in e\}$ to be the hyperedges incident to both vertices $u$ and $v$.
Let $d(v) = \sum_{e \in E(v)} \omega(e)$ denote the degree of vertex $v$, and let $\delta(e) = \sum_{v \in e} \gamma_e(v)$ denote the degree of hyperedge $e$. The \emph{vertex-weight matrix} $R$ of a hypergraph with edge-dependent vertex weights $H=(V, E, \omega, \gamma)$ is an $|E| \times |V|$ matrix with entries $R(e, v) = \gamma_e(v)$, and the \emph{hyperedge weight} matrix $W$ is a $|V| \times |E|$ matrix with $W(v, e) = \omega(e)$ if $v \in e$, and $W(v, e) = 0$ otherwise. The vertex-degree matrix $D_V$ is a $|V| \times |V|$ diagonal matrix with entries $D_V(v, v) = d(v)$, and the hyperedge-degree matrix $D_E$ is a $|E| \times |E|$ diagonal matrix with entries $D_E(e, e) = \delta(e)$. 

Given $H=(V, E, \omega, \gamma)$, the \textit{clique graph} of $H$, $G^H$, is an unweighted graph with vertices $V$, and edges $E' = \{(v, w) \subset V\times V : v, w \in e \textrm{ for some } e \in E\}$. In other words, $G^H$ turns all hyperedges into cliques. 

We say a hypergraph $H$ is \emph{connected} if its clique graph $G^H$ is connected. In this paper, we assume all hypergraphs are connected.

For a Markov chain with states $S$ transition probabilities $p$, we use $p_{u,v}$ to denote the probability of going from state $u$ to state $v$.

\section{Random Walks on Hypergraphs with Edge-Dependent Vertex Weights}
\label{compare}

Let $H=(V,E, \omega, \gamma)$ be a hypergraph with edge-dependent vertex weights. We first define a random walk on $H$. 
At time $t$, a random walker at vertex $v_t$ will do the following: 

\begin{figure*}[]
\centering
\includegraphics[scale=0.34]{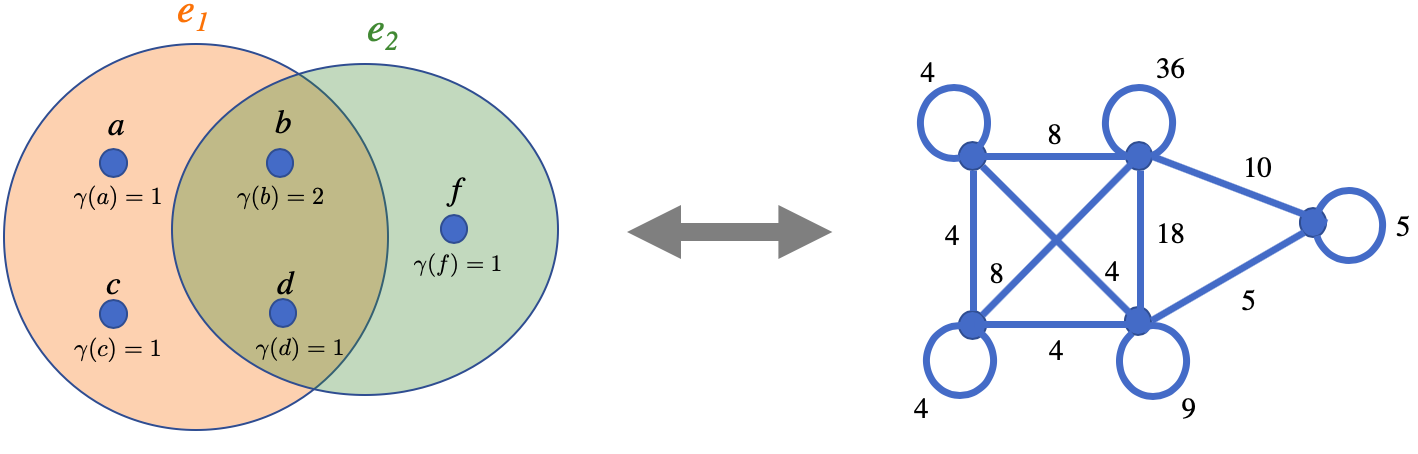}
\caption{Example illustrating Theorem \ref{vweights_thm}.  A hypergraph with \emph{edge-independent vertex weights} $H$ (left) and a corresponding edge-weighted clique graph $G^H$ (right) such that random walks on $H$ and $G^H$ are equivalent. Note that, if one changes the vertex weights of $b$ to be edge-\emph{dependent} vertex weights, by setting $\gamma_{e_1}(b) = 1, \gamma_{e_2}(b) = 2$, then it is not possible to choose edge weights $w_{u,v}$ on $G^H$ such that random walks on $G^H$ and $H$ are equivalent.}
\label{hg_ex2}
\end{figure*}

\begin{enumerate}
\item
Pick an edge $e$ containing $v$, with probability $\omega(e) / d(v)$.
\item
Pick a vertex $w$ from $e$, with probability $\gamma_e(w) / \delta(e)$.
\item
Move to vertex $v_{t+1} = w$, at time $t+1$.
\end{enumerate}

Formally, we define a random walk on $H$ by writing out the transition probabilities according to the above steps. 

\begin{defn}
A random walk on a hypergraph with edge-dependent vertex weights $H=(V,E, \omega, \gamma)$ is a Markov chain on $V$ with transition probabilities
\begin{equation}
p_{v, w} = \sum_{e \in E(v)} \left(\frac{\omega(e)}{d(v)} \right) \left( \frac{\gamma_e(w)}{\delta(e)}\right).
\end{equation}
\end{defn}

The \emph{probability transition matrix} $P$ of a random walk on $H$ is the $|V| \times |V|$ matrix with entries $P(v, w) = p_{v, w}$ and can be written in matrix form as $P=D_V^{-1} W D_E^{-1} R$.  (We use the convention that probability transition matrices have row sum $1$.) Using the probability transition matrix $P$, we can also define a random walk with restart on $H$ \cite{rwr}.
The random walk with restart is useful when it is unknown whether the random walk is irreducible.  

Note that our definition allows self-loops, i.e. $p_{v, v} > 0$, and thus the random walk is lazy.  While one can define a non-lazy random walk (i.e. $p_{v,v} = 0$
 for all $v$), the analysis of such walks is significantly more difficult, as the probability transition matrix cannot be factored as easily. In the Supplement, we show that a weaker version of Theorem \ref{vweights_thm} below holds for a non-lazy random walk.
 \citet{cover_times_hypergraph1} also studies the cover time of a non-lazy random walk on a hypergraph with edge-independent vertex weights. 

Next, we define what it means for two random walks to be equivalent. Because random walks are Markov chains, we define equivalence in terms of Markov chains.
\begin{defn}
\label{equiv_defn}
Let $M$ and $N$ be Markov chains with the same (countable) state space, and let $P^M$ and $P^N$ be their respective probability transition matrices. We say that $M$ and $N$ are \textbf{equivalent} if
\begin{equation*}
P^{M}_{x,y} = P^N_{x,y}
\end{equation*}
for all states $x$ and $y$.
\end{defn}

Using this definition, we state our first main theorem: a random walk on a hypergraph with \emph{edge-independent vertex weights} is equivalent to a random walk on its clique graph, for some choice of weights on the clique graph.
 \begin{theorem}
\label{vweights_thm}
Let $H=(V, E, \omega, \gamma)$ be a hypergraph with edge-\emph{independent} vertex weights. There exist weights $w_{u,v}$ on the clique graph $G^H$ such that a random walk on $H$ is equivalent to a random walk on $G^H$.
\end{theorem}

Theorem \ref{vweights_thm} generalizes the result by \citet{agarwal_higher_order}
who showed that the two hypergraph Laplacian matrices constructed in \citet{zhou} and \citet{rodriguez}  are equal to the Laplacian matrix of either the clique graph or the star graph, another graph constructed from a hypergraph.
\citet{agarwal_higher_order} also showed that the Laplacians of the clique graph and the star graph are equal when $H$ is $k$-uniform (i.e. when all hyperedges have size $k$), and are very close otherwise. Since the Laplacian matrices in \citet{zhou} and \citet{rodriguez} are derived from random walks on edge-independent vertex weights, Theorem \ref{vweights_thm} implies that both Laplacians are equal to the Laplacian of the clique graph -- even when the hypergraph is not $k$-uniform -- thus strengthening the result in \citet{agarwal_higher_order}.

The proof of Theorem \ref{vweights_thm} relies on the fact that a random walk on $H$ satisfies a property known as \emph{time-reversibility}: $\pi_u p_{u,v} = \pi_v p_{v, u}$ for all vertices $u, v \in V$, where $\pi$ is the stationary distribution of the random walk \cite{reversible}. It is well-known that a Markov chain can be represented as a random walk on a graph if and only if it is time-reversible. Moreover, time-reversiblility allows us to derive a formula for the weights $w_{u,v}$ on $G^H$.  Let $\gamma(v) = \gamma_e(v)$ be the edge-independent weight for vertex $v$.  Then,
\begin{equation}
\label{vweights}
w_{u,v} = \pi_u p_{u,v} = \sum_{e \in E(u, v)} \frac{\omega(e)\gamma(u) \gamma(v)}{\delta(e)}.
\end{equation}

Conversely, the caption of Figure \ref{hg_ex2} describes a simple example of a hypergraph with edge-dependent vertex weights that is not time-reversible. This proves the following result.

\begin{theorem}
\label{thm:hg_gh_neq}
There exists a hypergraph with edge-dependent weights $H=(V,E, \omega, \gamma)$ such that a random walk on $H$ is not equivalent to a random walk on its clique graph $G^H$ for any choice of edge weights on $G^H$.
\end{theorem}

Anecdotally, we find from simulations that most random walks on hypergraphs with edge-dependent vertex weights are not time-reversible, and therefore satisfy Theorem \ref{thm:hg_gh_neq}. However, it is not clear how to formalize this observation.


Theorem \ref{thm:hg_gh_neq} says that random walks on graphs with vertex set $V$ are a \emph{strict} subset of Markov chains on  $V$.
A natural follow-up question is whether \emph{all} Markov chains on $V$ can be described as a random walk on some hypergraph $H$ with vertex set $V$ and edge-dependent vertex weights. In the Supplement, we show that the answer to this question is no and provide a counterexample.

In addition, we show in the Supplement that hypergraphs with edge-dependent vertex weights create a rich hierarchy of Markov chains, beyond the division between time-reversible and time-irreversible Markov chains.  In particular, we show that random walks on  hypergraphs with edge dependent vertex weights and at least one hyperedge of cardinality $k$ cannot in general be reduced to a random walk on a hypergraph with hyperedges of cardinality at most $k-1$.

Finally, note that our definition of equivalent random walks (Definition \ref{equiv_defn}) requires the probability transition matrices to be equal.  Thus, another natural question is: given $H=(V,E,\omega,\gamma)$, do there exist weights on the clique graph $G^H$ such that random walks on $H$ and $G^H$ are ``close"? We provide a partial answer to this question in Section \ref{laplac}, where we show that, for a specific choice of weights on $G^H$, the second-smallest eigenvalues of the Laplacian matrices of $H$ and $G^H$ are close. 

} 

\section{Stationary Distribution and Mixing Time}
\label{stat_dist_mix}

\subsection{Stationary Distribution}

Recall the formula for the stationary distribution of a random walk on a graph. If $G=(V,E,w)$ is a graph, then the stationary distribution $\pi$ of a random walk on $G$ is
\begin{equation}
\label{stat_dist_sg}
\pi_v = \rho \sum_{e \in E(v)} w(e),
\end{equation}
where $\rho = \big(2 \sum_{e \in E} w(e)\big)^{-1}$.
We derive a formula for the stationary distribution for a random walk on a hypergraph with edge-dependent vertex weights; the formula is analogous to equation \eqref{stat_dist_sg} above with two important changes: first, the proportionality constant $\rho$ depends on the hyperedge, and second, each term in the sum is multiplied by the vertex weight  $\gamma_e(v)$.
\begin{theorem}
\label{stat_dist_thm}
Let $H=(V, E, \omega, \gamma)$ be a hypergraph with edge-independent vertex weights. There exist positive constants $\rho_e$ such that the stationary distribution $\pi$ of a random walk on $H$ is
\begin{equation}
\label{stat_dist_hg}
\pi_v = \sum_{e \in E(v)} \rho_e \omega(e)  \gamma_e(v).
\end{equation}
Moreover, $\rho_e$ can be computed in time $O\big(|E|^3 + |E|^2 \cdot |V|\big)$.

\end{theorem}

Note that while the vertex weights $\gamma_e(v)$ can be scaled arbitrarily without affecting the properties of the random walk, Theorem \ref{stat_dist_thm} suggests that $\rho_e$ is the ``correct" scaling factor.

When the hypergraph has edge-independent vertex weights (i.e. $\gamma_e(v) = \gamma(v)$ for all incident hyperedges $e$), $\rho_e = \left( \sum_{v \in V} \gamma(v) d(v)\right)^{-1}$, leading to the following formula for the stationary distribution:
\begin{equation}
\pi_v = \frac{d(v) \gamma(v)}{\sum_{v \in V}  d(v) \gamma(v)}.
\end{equation}

Furthermore, if the vertex weights are trivial (i.e. $\gamma(v) = 1$) then $\pi_v = d(v) / \sum_{v \in V} d(v)$, recovering the formula derived in \citet{zhou} for the stationary distribution of hypergraphs with trivial vertex weights.


\subsection{Mixing Time}
In this section, we derive a bound on the mixing time of a random walk on $H=(V, E, \omega, \gamma)$. First, we recall the definition of the mixing time of a Markov chain.

\begin{defn}
\label{mixing_time_def}
Let $M$ be a Markov chain with states $S$ and probability transition matrix $P$. The \emph{mixing time} of $M$ is
\begin{equation*}
t_{mix}(\epsilon) = \min\{ t \geq 0 : ||P^t(s, \cdot) - \pi||_{TV} \leq \epsilon, \forall s \in S \},
\end{equation*}
where $||\cdot||_{TV}$ is the total variation distance.
\end{defn}

We derive the following bound on the mixing time for a random walk on a hypergraph with edge-dependent vertex weights.

\begin{theorem}
\label{mix_time_hg}
Let $H=(V, E, \omega, \gamma)$ be a hypergraph with edge-dependent vertex weights. Without loss of generality, assume $\rho_e = 1$ (i.e. by multiplying the vertex weights in hyperedge $e$ by $\rho_e$).  
Then,
\begin{equation}
\label{mix_hg_eq}
t_{mix}^H(\epsilon) \leq \left\lceil \frac{8\beta_1}{\Phi^2} \log\left( \frac{1}{2\epsilon \sqrt{d_{min} \beta_2}} \right) \right\rceil,
\end{equation}
where
\begin{itemize}
\item $\Phi$ is the Cheeger constant of a random walk on $H$ \cite{montenegro, jerison}
\item
$d_{min}$ is the minimum degree of a vertex in $H$, i.e. $d_{min} = \min_v d(v)$,
\item
$\beta_1 = \displaystyle\min_{e \in E, v \in e} \displaystyle \left( \frac{\gamma_e(v)}{\delta(e)}\right)$,
\item
$\beta_2 = \displaystyle\min_{e \in E, v \in e} \big(\gamma_e(v)\big)$. 
\end{itemize}
\end{theorem}
This bound on the mixing time of the hypergraph random walk
has a similar form to the bound on the mixing time bound for a random walk on a graph \cite{jerison}. For a graph $G$ with edge weights $w(e)$ satisfying $\sum_v d(v) = 1$, we have,
\begin{equation}
\label{mix_sg_eq}
t_{mix}^{G}(\epsilon) \leq \left\lceil \frac{2}{\Phi^2} \log\left( \frac{1}{2\epsilon\sqrt{d_{min}}} \right) \right\rceil.
\end{equation}

Note that both $t_{mix}^H(\epsilon)$ and $t_{mix}^G(\epsilon)$ have the same dependence on $1/\Phi^2, \log(1/\epsilon)$, and $\log(1/\sqrt{d_{min}})$. Intuitively, the additional dependence of $t_{mix}^H(\epsilon)$ on $\beta_1$ and $\beta_2$ is because small values of $\beta_1$ and $\beta_2$ correspond to the hypergraph having vertices that are hard to reach, and the presence of such vertices increases the mixing time.

\section{Hypergraph Laplacian}
\label{laplac}

Let $H=(V,E, \omega, \gamma)$ be a hypergraph with edge-dependent vertex weights. 
Since a random walk on $H$ is a Markov chain, we can model the transition probabilities $p^H_{u,v}$ of the random walk using a weighted  \emph{directed} graph $G$ with the same vertex set $V$.
Specifically, let $G = (V,E', w')$ be a directed graph with directed edges $E'=\{(u, v) : \exists\; e \in E \text{ with } u, v \in e\}$, and edge weights $w'_{u,v} = p^H_{u,v}$.
Extending the definition of the Laplacian matrix for directed graphs \cite{chung_dg}, we define a Laplacian matrix $L$ for the hypergraph $H$ as follows.


\begin{defn}[Random walk-based hypergraph Laplacian]
\label{hg_laplac}
Let $H=(V,E, \omega, \gamma)$ be a hypergraph with edge-dependent vertex weights.  Let  $P$ be the probability transition matrix of a random walk on $H$ with stationary distribution $\pi$.  Let $\Pi$ be a $|V| \times |V|$ diagonal matrix with $\Pi_{v, v} = \pi_v$. Then, the random walk-based hypergraph Laplacian matrix $L$ is
\begin{equation}
\label{hg_lap}
L = \Pi - \frac{\Pi P + P^T \Pi}{2}.
\end{equation}
\end{defn}

At first glance, one might hypothesize that the hypergraph Laplacian  $L$ defined above does not model higher-order relations between vertices, since $L$ is defined using a directed graph containing edges only between pairs of vertices. Indeed, if $H$ has edge-independent vertex weights, then it is true that $L$ does not model higher-order relations between vertices.
This is because the transition probabilities $p^H_{u,v}$ are completely determined by the edge weights of the \emph{undirected} clique graph $G^H$ (Theorem \ref{vweights_thm}). Thus, for each pair $(u,v)$ of vertices in $H$, only a single quantity $w_{u,v}$, which encodes a pairwise relation between $u$ and $v$, is required to define the random walk. As such, the Laplacian matrix $L$ defined in Equation \eqref{hg_lap} is equal to the Laplacian matrix of an undirected graph, showing that $L$ only encodes pairwise relationships between vertices.

In contrast, when $H$ has edge-dependent vertex weights, the transition probabilities $p^H_{u,v}$ generally cannot be computed from a single quantity $w_{u,v}$ defined for each pair $(u,v)$ of vertices (Theorem \ref{thm:hg_gh_neq}).
The absence of such a reduction implies that the transition probabilities $p^H_{u,v}$, which are the edge weights of the directed graph $G'$, encode higher-order relations between vertices. Thus, the Laplacian matrix $L$ also encodes these higher-order relations.


From \citet{chung_dg}, the hypergraph Laplacian matrix $L$ given in equation \eqref{hg_lap} is positive semi-definite and has a Rayleigh quotient for computing its eigenvalues. $L$ can be used in developing spectral learning algorithms for hypergraphs with edge-dependent vertex weights, or to study the properties of random walks on such hypergraphs.
For example, the following Cheeger inequality for hypergraphs follows directly from the Cheeger inequality for directed graphs \cite{chung_dg}.


\begin{theorem}[Cheeger inequality for hypergraphs]
\label{cheeger_hg_thm}
Let $H=(V,E,\omega,\gamma)$ be a hypergraph with edge-dependent vertex weights. Let $L$ be the Laplacian matrix given in equation \eqref{hg_lap}, and let $\Phi$ be the Cheeger constant of a random walk on $H$. Let $\lambda_i$ be the non-zero eigenvalues of $L$, and let $\lambda = \min_i \lambda_i$. We have
\begin{equation}
\frac{\Phi^2}{2} \leq \lambda \leq 2\Phi.
\end{equation}
\end{theorem}


\subsection{Approximating the Hypergraph Laplacian with a Graph Laplacian}

In Section \ref{compare}, we posed the following question: given a hypergraph $H$ with edge-dependent vertex weights, can we find weights on the clique graph $G^H$ such that the random walks of $H$ and $G$ are close? 
We prove the following result.
\begin{theorem}
\label{hg_sg_eig}
Let $H=(V,E,\omega,\gamma)$ be a hypergraph, with the edge-dependent vertex weights normalized so that $\rho_e = 1$ for all hyperedges $e$. Let $G^H$ be the clique graph of $H$, with edge weights
\begin{equation}
\label{hg_sg_eig_weights}
w_{u, v} = \sum_{e \in E(u,v)} \frac{\omega(e) \gamma_e(u) \gamma_e(v)}{\delta(e)}.
\end{equation}
Let $L^H, L^G$ be the Laplacians of $H$ and $G^H$, respectively, and let $\lambda_1^H, \lambda_1^G$ be the second-smallest eigenvalues of $L^H, L^G$, respectively. Then
\begin{equation}
\frac{1}{c(H)} \lambda_1^H \leq \lambda_1^G \leq c(H) \lambda_1^H,
\end{equation}
where $\displaystyle c(H) = \max_{v \in V} \left( \frac{\max_{e \in E} \gamma_e(v)}{\min_{e \in E} \gamma_e(v)} \right)$.
\end{theorem}
This theorem says that there exist edge weights $w_{u,v}$ on $G^H$ such that second smallest eigenvalues of the Laplacians of $H$ and $G^H$ are within a constant factor $c(H)$ of each other, where $c(H)$ is determined by the vertex weights.
We do not know if the edge weights in Equation \eqref{hg_sg_eig_weights} give the tightest bound, or if another choice of edge weights on $G^H$ will yield a Laplacian $L^G$ that is ``closer" to the hypergraph Laplacians $L^H$.

Interestingly, \citet{gao_dynamic_hg} use a variant of $L^G$ as the Laplacian matrix of a hypergraph with edge-dependent vertex weights, and obtain state-of-the-art results on an object classification task.
Theorem \ref{hg_sg_eig} provides some theoretical evidence for why \citet{gao_dynamic_hg} are able to obtain good results, even with the ``wrong" Laplacian. 

\section{Experiments}
\label{exp}

We demonstrate the utility of hypergraphs with edge-dependent vertex weights in two different ranking applications: ranking authors in an academic citation network, and ranking players in a video game.

\subsection{Citation Network}



We construct a citation network of all machine learning papers from NIPS, ICML, KDD, IJCAI, UAI, ICLR, and COLT published on or before 10/27/2017, and extracted from the ArnetMiner database \cite{arnetminer}.
We represent the network as a hypergraph whose vertices $V$ are authors and whose hyperedges $E$ are papers, such that each hyperedge $e$ connects the authors of a paper.
The hypergraph has $|V|=28551$ vertices and $|E| = 25423$ hyperedges.

We consider two vertex weighted hypergraphs: $H_T = (V, E, \omega, \mathbf{1})$ has trivial vertex weights with $\gamma_e(v) =  1$ for all for all vertices $v$ and incident hyperedges $e$, and
$H_D = (V, E, \omega,\gamma_e)$ has edge-dependent vertex weights
\begin{equation*}
\gamma_e(v) = \begin{cases}
2 & \text{if vertex $v$ is the first or last author of paper,}\\
1 & \text{if vertex $v$ is a middle author of paper}.
\end{cases}
\end{equation*}
The edge-dependent vertex weights $\gamma_e(v)$ model unequal contributions by different authors.  For papers whose authors are in  alphabetical order (as is common in theory papers), we set vertex weights $\gamma_e(v) = 1$ for all $v \in e$.
We set the hyperedge weights $\omega(e) = (\text{number of citations for paper $e$}) + 1$ in both hypergraphs.

We calculate the stationary distribution of a random walk with restart on both $H_T$ and $H_D$ (restart parameter $\beta=0.4$), and rank authors $v$ in each hypergraph by their value in the stationary distribution. This yields two different rankings of authors: one with  edge-independent vertex weights, and one with edge-dependent vertex weights.

The two rankings have a Kendall $\tau$ correlation coefficient \cite{kt} of $0.77$, indicating modest similarity. Examining individual authors, we typically see that authors who are first/last authors on their most cited papers have higher rankings in $H_D$ compared to $H_T$, e.g. Ian Goodfellow \cite{gans}.  In contrast, authors who are middle authors on their most cited papers have lower rankings in $H_D$ relative to their rankings in $H_T$.
Table \ref{top_citation} shows the authors with rank above $700$ in at least one of the two hypergraphs, and with the largest gain in rank in $H_D$ relative to $H_T$.

\begin{table}[h]
\centering
\begin{tabular}{|c|c|c|}
\hline
Name                & Rank in $H_T$ & Rank in $H_D$ \\ \hline
Richard Socher      & 687           & 382           \\ \hline
Zhongzhi Shi        & 543           & 304           \\ \hline
Daniel Rueckert     & 619           & 391           \\ \hline
Lars Schmidt-Thieme & 673           & 454           \\ \hline
Tat-Seng Chua       & 650           & 435           \\ \hline
Ian J. Goodfellow   & 612           & 413           \\ \hline
\end{tabular}
\label{top_citation}
\caption{Highly ranked authors with the largest increase in rank when edge-dependent vertex weights are used in the hypergraph citation network.}
\end{table}

We emphasize that this example is intended to illustrate how a straightforward application of vertex weights leads to alternative author rankings. We do not anticipate that our simple scheme for choosing edge-dependent vertex weights will always yield the best results in practice.  For example, Christopher Manning drops in rank when edge-dependent vertex weights are added, but this is because he is the second-to-last, and co-corresponding, author on his most cited papers in the database.  A more robust vertex weighting scheme would include knowledge of such equal-contribution authors, and would also incorporate different relative contributions of first, middle, and corresponding authors.



\subsection{Rank Aggregation}
We illustrate the usage of hypergraphs with edge-dependent vertex weights on the \emph{rank aggregation problem}.
The rank aggregation problem aims to combine many partial rankings into one complete ranking. Formally, given a universe $\{1, 2, ..., n\}$ of items and a collection of partial rankings $\tau_1, ..., \tau_k$ (e.g. $\tau_i = (3, 1, 5)$ is a partial ranking expressing item $3 < $ item $1 < $ item $5$), a rank aggregation algorithm should find a permutation $\sigma$ on $\{1, 2, ..., n\}$ that is ``close" to the partial rankings $\tau_i$.

We consider a particular application of rank aggregation: ranking players in a multiplayer game. Here, 
the outcome of a game/match gives a partial ranking $\tau$ of the players participating in the match.  In addition to the ranking, one may also have additional information such as the scores of each player in the match.  The latter setting has been extensively studied; classic ranking methods are the ELO \cite{elo}, and Glicko \cite{glicko} systems that are used to rank chess players.  More recently, online multiplayer games such as Halo have led to the development of alternative ranking systems such as Microsoft's TrueSkill  \cite{trueskill} and TrueSkill 2 \cite{trueskill2}. 

We develop a rank aggregation algorithm that uses random walks on hypergraphs with edge-dependent vertex weights, and evaluate the performance of this algorithm on a real-world datasets of Halo 2 games. In the Supplement, we also include results on experiments with synthetic data.

\textbf{Data.} We analyze the Halo 2 dataset from the TrueSkill paper \cite{trueskill}.  This dataset contains two kinds of matches: free-for-all matches with up to $8$ players, and 1-v-1 matches. There are $31028$ free-for-all matches and $5093$ 1-v-1 matches among $5507$ players.  Using the free-for-all matches as partial rankings, we construct rankings of all players in the dataset, and evaluate those rankings on the 1-v-1 matches.  

\textbf{Methods.}
A well-known class of rank aggregation algorithms are Markov chain-based algorithms, first developed by \citet{dwork}. Markov-chain based algorithms create a Markov chain $M$ whose states are the players 
and whose the transition probabilities depend in some way on the partial rankings. The final ranking of players is determined by sorting the values in the stationary distribution $\pi$ of $M$.
In our experiments, we use a random walk with restart ($\beta=0.4$) instead of just a random walk, so that the stationary distribution always exists \cite{rwr}.

Using the free-for-all matches, we construct rankings of the players using four algorithms.  The first three algorithms use Markov chains: a random walk on hypergraph $H$ with edge-dependent vertex weights; a random walk on a clique graph; and \emph{MC3}, a Markov chain-based rank aggregation algorithm designed by \citet{dwork}.  The fourth algorithm is 
TrueSkill \cite{trueskill}.

First, we derive a rank aggregation algorithm using a random walk on a hypergraph $H=(V, E, \omega, \gamma)$ with edge-dependent vertex weights.   The vertices $V$ are the players, and the hyperedges $E$ correspond to the free-for-all matches. We set the hyperedge and vertex weights to be
\begin{align*}
\omega(e) &= (\text{standard deviation of scores in match $e$}) + 1, \\
\gamma_{e}(v) &= \exp[(\text{score of player $v$ in match $e$})].
\end{align*}

This choice of hyperedge weights are inspired by \citet{Ding2010}, who also use variance to define the hyperedge weights of their hypergraph. For vertex weights, we use $\exp(\text{score})$.  We choose these vertex weights instead of raw scores for two reasons: first, scores in Halo 5 can be negative, but vertex weights should be positive, and second, exponentiating the score gives more importance to the winner of a match. 
We chose to use relatively simple formulas for the hyperedge and vertex weights to evaluate the potential benefits of utilizing edge-dependent vertex weights; further optimization of vertex and edge weights may yield better performance.

Second, we derive a rank aggregation algorithm using a random walk on the clique graph $G^H$ of hypergraph $H$ described above, with the edge weights of $G^H$ given by Equation \ref{hg_sg_eig_weights}. Specifically, if $H=(V,E,\omega, \gamma)$ is the hypergraph defined above, then $G^H$ is a graph with vertex set $V$ and edge weights $w_{u,v}$ defined by
\begin{equation}
w_{u,v} = \sum_{e \in E(u,v)} \frac{\omega(e) \gamma_e(u) \gamma_e(v)}{\delta(e)}.
\end{equation}
In contrast to Equation \ref{hg_sg_eig_weights}, here we do not normalize vertex weights on $H$ so that $\rho_e = 1$ for each hyperedge $e$, since computing $\rho_e$ is computationally infeasible on our large dataset. Instead, we normalize vertex weights so that $\delta(e) = 1$ for all hyperedges $e$.


Third, we use \emph{MC3}, a Markov chain-based rank aggregation algorithm designed by \citet{dwork}. MC3 uses the partial rankings in each match; it does not use the score information. MC3 is very similar to a random walk on a hypergraph with edge-independent vertex weights. We convert the scores from each player in match $i$ into a partial ranking $\tau_i$ of the players, and use the $\tau_i$ as input to MC3.

Fourth, we use TrueSkill \cite{trueskill}. TrueSkill models each player's skill with a normal distribution. We rank players according to the mean of this distribution. We also implemented the probabilistic decision procedure for ranking players from the TrueSkill paper, and found no difference in performance between ranking by the mean of the distribution and the probabilistic decision procedure.

%

\textbf{Evaluation and Results: } We evaluate the rankings of each algorithm by using them to predict the outcomes of the 1-v-1 matches.  Specifically, given a ranking $\pi$ of players, we predict that the winner of a match between two players 
is the player with the higher ranking in $\pi$. 
Table~\ref{tab:real_res} shows the fraction of 1-v-1 matches correctly predicted by each of the four algorithms.
Random walks on the hypergraph with edge-dependent vertex weights have significantly better performance than both MC3 and random walks on the clique graph $G^H$, and comparable performance to TrueSkill. Moreover, on $8.9\%$ of 1-v-1 matches, the hypergraph method correctly predicts the outcome of the match, while TrueSkill incorrectly predicts the outcome---suggesting that the hypergraph model is capturing some information about the players that TrueSkill is missing.  Unfortunately, we are unable to identify any specific pattern in the matches where the hypergraph predicted the outcome correctly and TrueSkill predicted incorrectly.

\begin{table}[h]
\centering
\caption{Result of ranking players for Halo 2 Dataset.}
\vskip 0.15in
\begin{tabular}{|c|c|}
\hline
           & Correctly Predicted \\ \hline
TrueSkill  & 73.4\%               \\ \hline
Hypergraph & 71.1\%               \\ \hline
Clique Graph & 61.1\%               \\ \hline
MC3        & 52.3\%               \\ \hline
\end{tabular}
\label{tab:real_res}
\end{table}
\section{Conclusion}

In this paper, we use random walks to develop a spectral theory for hypergraphs with edge-dependent vertex weights. We demonstrate both theoretically and experimentally how edge-dependent vertex weights  model higher-order information in hypergraphs and improve the performance of hypergraph-based algorithms. At the same time, we show that random walks on hypergraphs with edge-independent vertex weights are equivalent to random walks on graphs, generalizing earlier results tha showed this equivalence in special cases \cite{agarwal_higher_order}.

There are numerous directions for future work.  It would be desirable to evaluate additional applications where hypergraphs with edge-dependent vertex weights have previously been used (e.g. \cite{gao_dynamic_hg, etail}), replacing the Laplacian used in some of these works with the  hypergraph Laplacian introduced in Section \ref{laplac}.  Sharper bounds on the approximation of the hypergraph Laplacian by a graph Laplacian are also desirable. Another direction is to examine the relationship between the linear hypergraph Laplacian matrix introduced here and the nonlinear Laplacian operators that were recently introduced in the case of trivial vertex weights \cite{louis_hg} or submodular vertex weights
\cite{panli_nips, panli_icml}.

Another interesting direction is in extending graph convolutional neural networks (GCNs) to hypergraphs.
Recent approaches to GCNs implement the graph convolution operator as a 
non-linear function of the graph Laplacian \cite{graph_cnn_kipf, graph_cnn_defferrard}.  GCNs have also been generalized to hypergraph convolutional neural networks (HGCNs), where the convolution layer operates on a hypergraph with edge-independent vertex weights instead of a graph 
\cite{hyper_cnn1, hyper_cnn2}.
The hypergraph Laplacian matrix introduced in this paper would allow one to extend HGCNs to hypergraphs with edge-dependent vertex weights.

\bibliography{hypergraphs}
\bibliographystyle{plainnat}

\newpage

\appendix

\section{Incorrect Stationary Distribution in Earlier Work}

\citet{etail} claim in Equation 4 that the stationary distribution $\pi$ of a random walk on a hypergraph $H=(V,E,\gamma,\omega)$ with edge-dependent vertex weights is
\begin{equation}
\label{eq:wrong_stat_dist}
\pi_v = \frac{d(v)}{\sum_{u \in V} d(u)},
\end{equation}
where $d(v) = \sum_{e \in E(v)} \omega(e)$ is the sum of edge weights of incident hyperedges.  Curiously, the stationary distribution given by this formula does not depend on the vertex weights.
A counterexample to this formula is shown in hypergraph $H$ in Figure 1 of the main text, with edge-dependent vertex weights as described in the caption (i.e. $\gamma_{e_1}(b) = 1, \gamma_{e_2}(b)=2$). Computing the stationary distribution $\pi$ of a random walk on $H$ yields that $\pi_b = 7/20$, while Equation \eqref{eq:wrong_stat_dist} incorrectly yields $\pi_b = 2/7$.

\section{Proof of Theorem \ref{vweights_thm}}

First we need the following definition and lemma.

\begin{defn}
Let $M$ be a Markov chain with state space $X$ and transition probabilities $p_{x,y}$, for $x, y \in S$. We say $M$ is \emph{reversible} if there exists a probability distribution $\pi$ over $S$ such that
\begin{equation}
\pi_x p_{x,y} = \pi_y p_{y,x}.
\end{equation}
\end{defn}

\begin{lem}
\label{db_lem}
Let $M$ be an irreducible Markov chain with finite state space $S$ and transition probabilities $p_{x,y}$ for $x, y, \in S$. 
$M$ is reversible if and only if there exists a weighted, undirected graph $G$ with vertex set $S$ such that a random walk on $G$ and  $M$ are equivalent.
\end{lem}

\begin{proof}[Proof of Lemma]
First, suppose $M$ is reversible. Since $M$ is irreducible, let $\pi$ be the stationary distribution of $M$. Note that, because $M$ is irreducible, $\pi_x \neq 0$ for all states $x$. 

Let $G$ be a graph with vertices $S$, and edge weights $w_{x,y} = \pi_x p_{x,y}$. By reversibility, $G$ is well-defined. In a random walk on $G$, the probability of going from $x$ to $y$ in one time-step is
\begin{equation*}
\frac{w_{x,y}}{\sum_{z \in S} w_{x,z}} = \frac{\pi_x p_{x,y}}{\sum_{z \in S} \pi_x p_{x,z}} = \frac{p_{x,y}}{\sum_{z \in S} p_{x,z}} = p_{x,y},
\end{equation*}
since $\sum_{z \in S} p_{x,z} = 1$.

Thus, if $M$ is reversible, the stated claim holds. The other direction follows from the fact that a random walk on an undirected graph is always reversible \cite{reversible}.
\end{proof}

\begin{customthm}{\ref{vweights_thm}}
Let $H=(V, E, \omega, \gamma)$ be a hypergraph with edge-\emph{independent} vertex weights. Then, there exist weights $w_{u,v}$ on the clique graph $G^H$ such that a random walk on $H$ is equivalent to a random walk on $G^H$.
\end{customthm}

\begin{proof}[Proof of Theorem \ref{vweights_thm}]

Let $\gamma(v) = \gamma_e(v)$ for vertices $v$ and incident hyperedges $e$. We first show that a random walk on $H$ is reversible.
By Kolmogorov's criterion, reversibility is equivalent to
\begin{equation}
p_{v_1,v_2} p_{v_2,v_3} \cdots p_{v_n,v_1} = p_{v_1,v_n} p_{v_n,v_{n-1}}\cdots p_{v_2,v_1}.
\end{equation}
for any set of vertices $v_1, \dots, v_n$.

Since the transition probabilities for any two vertices $u,v$ are 
\begin{equation}
p_{u,v} = \sum_{e \in E(u,v)} \frac{\omega(e)}{d(u)}\frac{\gamma(u)}{\delta(e)} = \frac{\gamma(u)}{\delta(u)} \sum_{e \in E(u,v)} \frac{\omega(e)}{\delta(e)},
\end{equation}
we have
\begin{equation}
\begin{split}
p_{v_1,v_2} p_{v_2,v_3} \cdots p_{v_n,v_1} &= \left( \frac{\gamma(v_1)}{\delta(v_1)} \sum_{e \in E(v_1, v_2)} \frac{\omega(e)}{\delta(e)} \right) \cdots \left( \frac{\gamma(v_n)}{\delta(v_n)} \sum_{e \in E(v_n, v_1)} \frac{\omega(e)}{\delta(e)} \right) \\
&= \prod_{i=1}^n \left( \frac{\gamma(v_i)}{\delta(v_i)} \sum_{e \in E(v_i, v_{i+1})} \frac{\omega(e)}{\delta(e)} \right), \text{where we define $v_{n+1}=v_1$} \\
&= \left( \frac{\gamma(v_1)}{\delta(v_1)} \sum_{e \in E(v_n, v_1)}, \frac{\omega(e)}{\delta(e)} \right) \cdots \left( \frac{\gamma(v_2)}{\delta(v_2)} \sum_{e \in E(v_2, v_1)} \frac{\omega(e)}{\delta(e)} \right) \\
&= p_{v_1,v_n} p_{v_n,v_{n-1}}\cdots p_{v_2,v_1}.
\end{split}
\end{equation}
So by Kolmogorov's criterion, a random walk on $H$ is reversible.

Furthermore, because $H$ is connected, random walks on $H$ are irreducible. Thus, by Lemma \ref{db_lem}, there exists a graph $G$ with vertex set $V$ and edge weights $w_{u,v}$ such that random walks on $G$ and $H$ are equivalent. The equivalence of the random walks implies that $p_{u,v} > 0$ if and only if $w_{u,v} > 0$, so it follows that $G$ is the clique graph of $H$.
%
%
\end{proof}

\section{Non-Lazy Random Walks on Hypergraphs}

First we generalize the random walk framework of \citet{cover_times_hypergraph1} to random walks on hypergraphs with edge-dependent vertex weights. Informally, in a non-lazy random walk, a random walker at vertex $v$ will do the following:
\begin{enumerate}
\item
pick an edge $e$ containing $v$, with probability $\frac{\omega(e)}{d(v)}$,
\item
pick a vertex $w \neq v$ from $e$, with probability $\frac{\gamma_e(w)}{\delta(e) - \gamma_e(v)}$, and
\item
move to vertex $w$.
\end{enumerate}

Formally, we have the following.
\begin{defn}
A \emph{non-lazy random walk} on a hypergraph with edge-dependent vertex weights $H=(V,E, \omega, \gamma)$ is a Markov chain on $V$ with transition probabilities
\begin{equation}
p_{v, w} = \sum_{e \in E(v)} \left(\frac{\omega(e)}{d(v)} \right) \left( \frac{\gamma_e(w)}{\delta(e) - \gamma_e(v)}\right).
\end{equation}
for all states $v \neq w$.
\end{defn}

It is also useful to define a modified version of the clique graph without self-loops.

\begin{defn}
Let $H=(V,E, \omega, \gamma)$ be a hypergraph with edge-dependent vertex weights. The \emph{clique graph of $H$ without self-loops}, $G^H_{nl}$, is a weighted, undirected graph with vertex set $V$, and edges $E'$ defined by
\begin{equation}
E' = \{(v, w) \in V \times V : v, w \in e \text{ for some } e \in E, \text{ and } v \neq w\}.
\end{equation}
\end{defn}

In contrast to the lazy random walk, a non-lazy random walk  on a hypergraph with edge-independent vertex weights is not guaranteed to satisfy reversibility. However, if $H$ has \emph{trivial} vertex weights, then reversibility holds, and we get the following result.

\begin{theorem}
Let $H=(V, E, \omega, \gamma)$ be a hypergraph with trivial vertex weights, i.e. $\gamma_e(v) = 1$ for all vertices $v$ and incident hyperedges $e$. Then, there exist weights $w_{u,v}$ on the clique graph without self-loops $G^H_{nl}$ such that a non-lazy random walk on $H$ is equivalent to a random walk on $G^H_{nl}$.
\end{theorem}

\begin{proof}
Again, we first show that a non-lazy random walk on $H$ is reversible. Define the probability mass function $\pi_v = c\cdot d(v)$ for normalizing constant $c>0$. Let $p_{u, v}$ be the probability of going from $u$ to $v$ in a non-lazy random walk on $H$, where $u \neq v$. Then,
\begin{equation*}
\begin{split}
\pi_u P_{u,v} &= c \cdot d(u) \cdot \left( \sum_{e \in E(u, v)} \frac{w(e)}{d(u)}\cdot \frac{1}{|e|-1} \right) \\[5pt]
&= \sum_{e \in E(u, v)} \left(\omega(e) \cdot \frac{c}{|e|-1} \right).
\end{split}
\end{equation*}

By symmetry, $\pi_up_{u,v}=\pi_v p_{v, u}$, so a non-lazy random is reversible. Thus, by Lemma \ref{db_lem}, there exists a graph $G$ with vertex set $V$ and edge weights $w_{u,v}$ such that a random walk on $G$ and a non-lazy random walk on $H$ are equivalent. The equivalence of the random walks implies that $p_{u,v} > 0$ if and only if $w_{u,v} > 0$, so it follows that $G$ is the clique graph of $H$ without self-loops.
\end{proof}

\section{Relationships between Random Walks on Hypergraphs and Markov Chains on Vertex Set}

In the main text, we show that there are hypergraphs with edge-dependent vertex weights whose random walks are not equivalent to a random walk on a graph. A natural follow-up question is to ask whether all Markov chains on a vertex set $V$ can be represented as a random walk on some hypergraph with the same vertex set and edge-dependent vertex weights. Below, we show that the answer is no. Since  random walks on hypergraphs with edge-dependent vertex weights are lazy, in the sense that $p_{v,v} > 0$ for all vertices $v$, we restrict our attention to lazy Markov chains with $p_{v,v} = 0$.

\begin{claim}
There exists a lazy Markov chain $M$ with state space $V$ such that $M$ is not equivalent to a random walk on a hypergraph with vertex set $V$ and edge-dependent vertex weights.
\end{claim}

\begin{proof}
Suppose for the sake of contradiction that any lazy Markov chain with $V$ is equivalent to a random walk on some hypergraph with vertex set $V$.
Let $M$ be a lazy Markov chain with states $V$ and transition probabilities $p^M$, with the following property. For some states $x, y \in V$, let
\begin{align}
\label{eq:mc_counter_ex}
\begin{split}
p^M_{x,x} &= 0.9 \\
p^M_{x,y} &= 0.01 \\
p^M_{y,x} &= 0.1 \\
p^M_{y,y} &= 0.001.
\end{split}
\end{align}

By assumption, let $H=(V,E,\omega, \gamma)$ be a hypergraph with vertex set $V$ and edge-dependent vertex weights, such that a random walk on $H$ is equivalent to $M$. Let $p^H$ be the transition probabilities of a random walk on $H$. We have
\begin{equation}
\begin{split}
d(x) \cdot p^M_{x,x} &= d(x) \cdot p^H_{x,x} \\
&= \sum_{e \in E(x)} \omega(e) \cdot \left( \frac{\gamma_e(x)}{\delta(x)}\right) \\
&\geq \sum_{e \in E(x,y)} \omega(e) \cdot \left( \frac{\gamma_e(x)}{\delta(x)}\right) \\
&= d(y) \cdot p^H_{y,x} \\
&= d(y) \cdot p^M_{y,x}
\end{split}
\end{equation}

Plugging in Equations \eqref{eq:mc_counter_ex} to the above yields $d(x) \cdot 0.9 \geq d(y) \cdot 0.1$, or $9 d(x) \geq d(y)$.

By similar reasoning, we also have $d(y) \cdot p^M_{y,y} \geq d(x) \cdot p^M_{x,y}$, and plugging in Equations \eqref{eq:mc_counter_ex} gives us $d(y) \cdot 0.001 \geq d(x) \cdot 0.01$, or $d(y) \geq 10 d(x)$.

Combining both of these inequalities, we obtain
\begin{equation}
9 d(x) \geq d(y) \geq 10 d(x).
\end{equation}
Since the vertex degree $d(x) \ge 0$, we obtain a contradiction.
\end{proof}

Next, for any $k>1$, define a $k$-hypergraph to be a hypergraph with edge-dependent vertex weights whose hyperedges have cardinality at most $k$. We show that, for any $k$, there exists a $k$-hypergraph with vertex set $V$ whose random walk is not equivalent to the random walk of any $(k-1)$-hypergraph with vertex set $V$. We first prove the result for $k=3$.

\begin{lem}
\label{lem:khg_lem}
There exists a $3$-hypergraph with vertex set $V$, whose random walk is not equivalent to a random walk on any $2$-hypergraph with vertex set $V$.
\end{lem}

\begin{proof}
Let $H_3 = (V, E_3, \omega, \gamma)$ be a $3$-hypergraph with four vertices, $V=\{v_1, v_2, v_3, v_4\}$, and two hyperedges $e_1 = \{v_1, v_2, v_3\}$ and $e_2 = \{v_1, v_3, v_4\}$. Let the hyperedge weights be $\omega(e_1) = \omega(e_2) = 1$ and the vertex weights be $\gamma_{e_1}(v_1) = 2$, and $\gamma_{e_i}(v_j) = 1$ for all other $v_j, e_i$ such that $v_j \in e_i$.

\begin{figure}[h]
\centering
\includegraphics[scale=0.3]{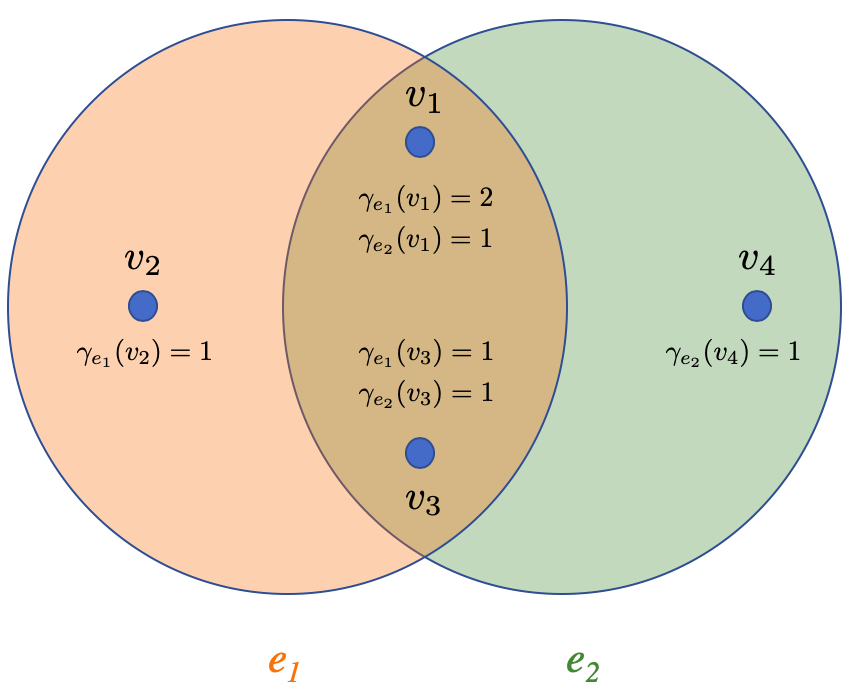}
\caption{Pictured above is $H_3$.}
\end{figure}

For the sake of contradiction, suppose a random walk on $H_3$ is equivalent to a random walk on $H_2 = (V, E_2, \omega, \gamma)$, where $H_2$ is a $2$-hypergraph with vertex set $V$. Let $p^{H_i}$ be the transition probabilities of $H_i$ for $i=2, 3$; by assumption, $p^{H_2} = p^{H_3}$.

$H_2$ must have the following edges: $e'_{12} = \{v_1, v_2\}$, $e'_{14} = \{v_1, v_4\}$, $e'_{23} = \{v_2, v_3\}$, $e'_{34} = \{v_3, v_4\}$, and $e'_{13} = \{v_1, v_3\}$. 
WLOG let $\gamma_{e_{ij}}(v_i) + \gamma_{e_{ij}}(v_j) = 1$ for each $i, j$.
Moreover, while we do not depict these edges in the figure below, $H_2$ also has edges $e'_i = \{v_i\}$ for $i=1, 2, 3, 4$, though it may be the case that $\omega(e'_i) = 0$. 

For shorthand, we write $\omega_{ij}$ for $\omega(e'_{ij})$, $\omega_i$ for $\omega(e'_i)$, and $\gamma_{ijk}$ for $\gamma_{e'_{ij}} (v_k)$ where $k\in\{i, j\}$.

\begin{figure}[h]
\centering
\includegraphics[scale=0.3]{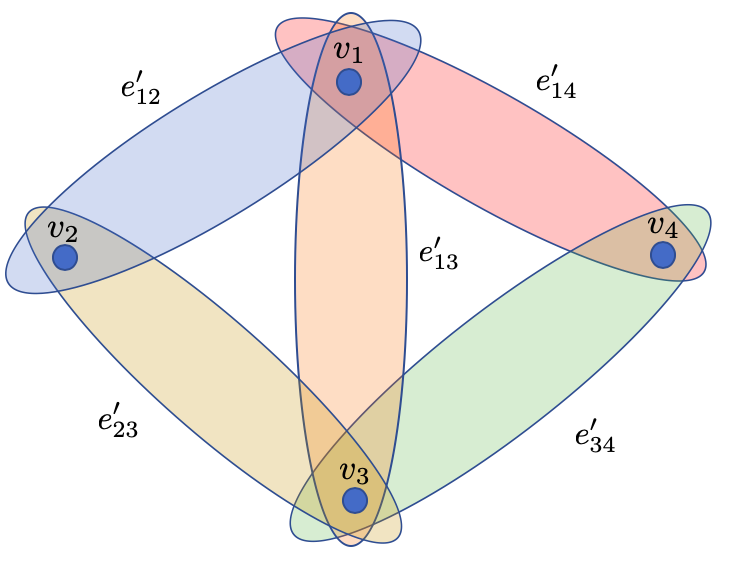}
\caption{Pictured above is $H_2$. For illustrative purposes, we do not draw out singleton edges.}
\end{figure}

By definition, we have
\begin{equation}
\frac{1}{2} = p^{H_3}_{v_2,v_1} = p^{H_2}_{v_2, v_1} = \left( \frac{\omega_{12}}{\omega_{12}+\omega_{23}+\omega_2} \right) \gamma^{121}
\end{equation}

Thus, $\left( \frac{\omega_{12}}{\omega_{12}+\omega_{23}+\omega_2} \right) = (2 \cdot \gamma_{121})^{-1}$. 

By similar analysis of $p^{H_3}_{v_2,v_3}$, and using that $\gamma_{232} + \gamma_{233} = 1$, we also have $\left( \frac{\omega_{23}}{\omega_{12}+\omega_{23}+\omega_2} \right) = \big(4 \big(1-\gamma_{232}\big)\big)^{-1}$. Thus, adding together the bounds on $p^{H_2}_{v_2,v_1}$ and $p^{H_3}_{v_2,v_1}$
\begin{equation}
\label{eq:2hg_1}
\frac{1}{2 \gamma_{121}} + \frac{1}{4 (1-\gamma_{232})} = \left( \frac{\omega_{12}}{\omega_{12}+\omega_{23}+\omega_2} \right) + \left( \frac{\omega_{23}}{\omega_{12}+\omega_{23}+\omega_2} \right) \leq 1.
\end{equation}

Note that, to get the bound in Equation \eqref{eq:2hg_1}, we summed $p^{H_2}_{v_2, v_i}$ for $i \neq 2$. If we follow the same steps but replace $v_2$ with $v_1, v_3$, we get the following bounds, respectively:
\begin{align}
\label{eq:2hg_2}
&\frac{1}{8\cdot \gamma_{121}} + \frac{7}{24(1-\gamma_{131})} + \frac{1}{6(1-\gamma_{141})} \leq 1 \\[0.5em]
\label{eq:2hg_3}
&\frac{1}{8 \gamma_{232}} + \frac{5}{12 \gamma_{131}} + \frac{1}{6 \gamma_{344}} \leq 1.
\end{align}

Now, solving for $\gamma_{121}$ in Equation \eqref{eq:2hg_1} yields
\begin{equation}
\label{eq:2hg_4}
\gamma_{121} \geq \frac{2(1-\gamma_{232})}{3-4\gamma_{232}}.
\end{equation}

Next, using that $\gamma_{ijk} \in [0,1]$, we bound Equation \eqref{eq:2hg_2}:
\begin{equation}
\label{eq:2hg_5}
\begin{split}
1 &\geq \frac{1}{8\gamma_{121}} + \frac{7}{24(1-\gamma_{131})} + \frac{1}{6(1-\gamma_{141})} \\[0.5em]
&\geq \frac{1}{8\gamma_{121}} + \frac{7}{24} + \frac{1}{6} \\[0.5em]
&= \frac{1}{8\gamma_{121}} + \frac{11}{24}.
\end{split}
\end{equation}

Solving for $\gamma_{121}$ yields $\gamma_{121} \leq \frac{10}{13}$. Combining with Equation \eqref{eq:2hg_4}:
\begin{equation}
\label{eq:2hg_6}
\frac{10}{13} \geq \gamma_{121} \geq \frac{2(1-\gamma_{232})}{3-4\gamma_{232}} \Longrightarrow \gamma_{232} \leq \frac{2}{7}.
\end{equation}

Bounding Equation \eqref{eq:2hg_3} in a similar way to Equation \eqref{eq:2hg_5} gives us:
\begin{equation}
\begin{split}
1 &\geq \frac{1}{8\gamma_{232}} + \frac{5}{12\gamma_{131}} + \frac{1}{6\gamma_{344}} \\[0.5em]
&\geq \frac{1}{8\gamma_{232}} + \frac{5}{12} + \frac{1}{6} \\[0.5em]
&= \frac{1}{8\gamma_{232}} + \frac{7}{12}.
\end{split}
\end{equation}
Solving for $\gamma_{232}$ gives us 
\begin{equation}
\label{eq:2hg_7}
\gamma_{232} \geq \frac{3}{10}.
\end{equation}
Finally, putting together Equations \eqref{eq:2hg_6} and \eqref{eq:2hg_7}:
\begin{equation}
\frac{3}{10} \leq \gamma_{232} \leq \frac{2}{7},
\end{equation}
which yields a contradiction, as $\frac{3}{10} > \frac{2}{7}$.
\end{proof}

We prove the result for general $k$ by extending the above proof.

\begin{theorem}
Let $k > 1$. Then, there exists a $k$-hypergraph with vertex set $V$ whose random walk is not equivalent to a random walk on any $(k-1)$-hypergraph with vertex set $V$.
\end{theorem}

\begin{proof}
For simplicity, assume $k$ is even (our argument can be adapted to odd $k$). Write $k=2(n+1)$. For the sake of contradiction, suppose all $k$-hypergraphs have random walks equivalent to the random walk of some $(k-1)$-hypergraph. 

Let $H_k=(V,E_k,\omega, \gamma)$ be a $k$-hypergraph with vertices $V=\{v_1, \dots, v_n, w_1, \dots, w_n, x, y\}$, and hyperedges $e_1 = \{v_1, \dots, v_n, b, c\}$ and $e_2 = \{w_1, \dots, w_n, b, c\}$. The edge weights are $\omega(e_1) = \omega(e_2) = 1$, and the edge-dependent vertex weights are $\omega_{e_1}(b) = 2$, and $\omega_{e_i}(v)=1$ for all other $v, e_i$ with $v \in e_i$.

\begin{figure}[ht]
\centering
\includegraphics[scale=0.3]{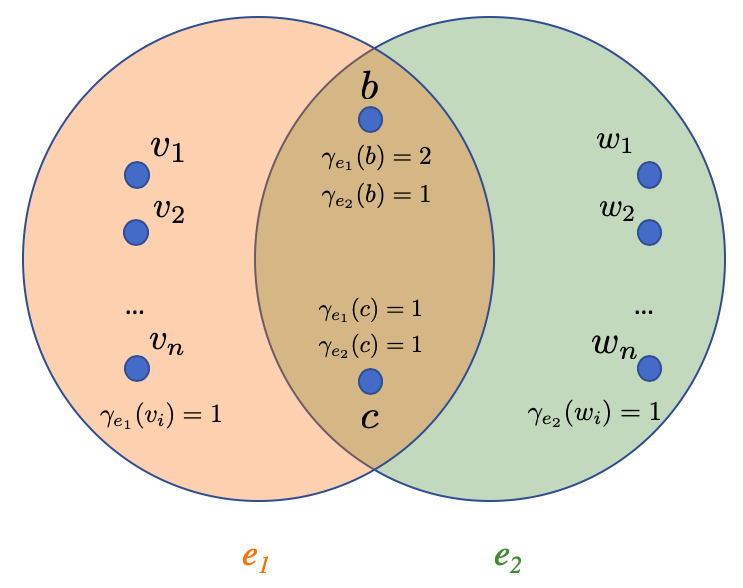}
\caption{Pictured above is $H_k$.}
\end{figure}

By assumption, let $H_{k-1}=(V,E_{k-1}, \omega, \gamma)$ be a $(k-1)$-hypergraph whose random walk is equivalent to a random walk on $H_k$. Let $p^{H_k}$, $p^{H_{k-1}}$ be the transition probabilities of $H_k, H_{k-1}$, respectively. 

Then, in $H_{k-1}$, we have
\begin{equation}
\label{eq:khg_1}
d(v_i) \cdot p^{H_{k-1}}_{v_i,v_j} = \sum_{e \in E(v_i, v_j)} \omega(e) \cdot \left( \frac{\gamma_e(v_j)}{\delta(e)} \right) \leq \sum_{e \in E(v_j)} \omega(e) \cdot \left( \frac{\gamma_e(v_j)}{\delta(e)} \right) = d(v_j) \cdot  p^{H_{k-1}}_{v_j,v_j} 
\end{equation}
for all $i, j \in \{1, \dots, n\}$. Since $p^{H_{k-1}}_{v_i,v_j} = p^{H_{k-1}}_{v_j,v_j}$, the above equation implies $d(v_i) \leq d(v_j)$. So by symmetry, $d(v_i) = d(v_j)$ for all $i, j$.

This means that Equation \eqref{eq:khg_1} is actually a strict equality, so 
\begin{equation}
\sum_{e \in E(v_i, v_j)} \omega(e) \cdot \left( \frac{\gamma_e(v_j)}{\delta(e)} \right) = \sum_{e \in E(v_j)} \omega(e) \cdot \left( \frac{\gamma_e(v_j)}{\delta(e)} \right).
\end{equation}
Since every term in the above sums are positive and equal, it must be the case that every hyperedge in $H_{k-1}$ containing $v_j$ also contains $v_i$, for all $i, j$. Because they all are in the same hyperedges in both $H_{k-1}$ and $H_k$, we can view $\{v_1, \dots, v_n\}$ as a single ``supernode" $v$. By symmetry, we can also view $\{w_1, \dots, w_n\}$ as a single supernode $w$.

Thus, we have reduced our problem to the counterexample in Lemma \ref{lem:khg_lem}, and the result follows.
\end{proof}

Putting all of our results together gives us the following (informal) hierarchy of Markov chains
\begin{equation*}
\begin{split}
\{ \text{random walks on hypergraphs with edge-independent vertex weights} \} &= \{\text{random walks on graphs}\} \\
&\subsetneq \{\text{random walks on $2$-hypergraphs}\} \\
&\subsetneq \{\text{random walks on $3$-hypergraphs}\} \\
&\subsetneq \dots \\
&\subsetneq \{\text{all lazy Markov chains}\}.
\end{split}
\end{equation*}

\section{Proof of Theorem \ref{stat_dist_thm}}

We first prove the following lemma.

\begin{lem}
\label{stat_dist_lem}
Let $H=(V,E)$ be a hypergraph with edge-dependent vertex weights $\gamma_e(v)$ and hyperedge weights $\omega(e)$. Without loss of generality, assume $\sum_{v \in e} \gamma_e(v) = 1$. There exist $\rho_e > 0$ satisfying
\begin{equation}
\label{stat_dist_lem_eq_1}
\rho_e = \sum_{v \in e} \sum_{f \in E(v)} d(v)^{-1} \cdot \rho_f  \cdot \omega(f) \cdot \gamma_f(v)
\end{equation}
and
\begin{equation}
\label{stat_dist_lem_eq_2}
\sum_{e \in E} \rho_e \cdot \omega(e) = 1.
\end{equation}
\end{lem}

\begin{proof}[Proof of Lemma]
Our proof outline is as follows. First, we prove the lemma in the case where the hyperedge weights are all equal to each other. Then, we extend that result to the case where the hyperedge weights are rational. Finally, we use the density of $\mathbb{Q}$ in $\mathbb{R}$ to extend our result from rational hyperedge weights to real ones.

First, suppose all of the hyperedge weights are equal to each other. WLOG let $\omega(e) = 1$ for all $e \in E$. Switching the order of summation in Equation \ref{stat_dist_lem_eq_1}, we have
\begin{equation}
\label{stat_dist_lem_proof_eq_1}
\begin{split}
\sum_{v \in e} \sum_{f \in E(v)} d(v)^{-1} \cdot \rho_f  \cdot \omega(f) \cdot \gamma_f(v) &= \sum_{v \in e} \sum_{f \in E(v)} d(v)^{-1} \cdot \rho_f  \cdot \gamma_f(v) \\
&= \sum_{f \in E} \sum_{v \in e\cap f} d(v)^{-1} \cdot \rho_f \cdot \gamma_f(v) \\
 &= \sum_{f \in E} \rho_f \cdot \left( \sum_{v \in e \cap f} d(v)^{-1} \gamma_f(v) \right).
 \end{split}
\end{equation}
 
Now let $A$ be a square matrix of size $|E| \times |E|$, with entries $A_{e,f} = \sum_{v \in e \cap f} d(v)^{-1} \gamma_f(v)$. Note that the column sums of $A$ are equal to $1$:
\begin{equation}
\begin{split}
\sum_{e\in E} A_{e, f} &= \sum_{e \in E} \sum_{v \in e \cap f} d(v)^{-1} \gamma_f(v) \\
&= \sum_{v \in f} \sum_{e \in E(v)} d(v)^{-1} \gamma_f(v) \\
&= \sum_{v \in f} d(v)^{-1} \gamma_f(v)\cdot d(v) \\
&= \sum_{v \in f} \gamma_f(v) \\
&= 1.
\end{split}
\end{equation}

Thus, by the Perron-Frobenius theorem, $A$ has a positive eigenvector $\rho$ with eigenvalue $1$. 

So by construction, $\rho$ satisfies Equation \ref{stat_dist_lem_eq_1}. Moreover, $t\cdot \rho$ also satisfies Equation \ref{stat_dist_lem_eq_1} for any $t>0$. Thus, $t\cdot \rho$ with $t=\left(\sum_{e\in E}\rho_e\cdot \omega(e)\right)^{-1}$ satisfies both Equation \ref{stat_dist_lem_eq_1} and Equation \ref{stat_dist_lem_eq_2}, and so the lemma is proved in the case where the hyperedge weights are all equal. 

Next, assume $H$ is a hypergraph with rational hyperedge weights, i.e. $\omega(e) \in \mathbb{Q}$ for all $e \in E$. Multiplying through by denominators, we can assume $\omega(e) \in \mathbb{N}$. Create hypergraph $H'$ with vertices $V$ in the following way. For each hyperedge $e$, replace $e$ with hyperedges $e_1, ..., e_{\omega(e)}$, where each hyperedge $e_i$:
\begin{itemize}
\item
contains the same vertices as $e$, 
\item
has weight $\omega'(e_i)=1$,
\item 
has the same vertex weights as $e$, so that $\gamma'_{e_i}(v) = \gamma_e (v)$ for all $v \in e$.
\end{itemize}

Let $E'$ be the hyperedges of $H'$, and let $M(v)$ be the hyperedges incident to vertex $v$ in $H'$. Since $H'$ has equal hyperedge weights, we can find constants $\rho'_{e_i}$ that satisfy Equations \ref{stat_dist_lem_eq_1} and \ref{stat_dist_lem_eq_2} for $H'$. Note that $\rho'_{e_i} = \rho'_{e_j}$ by symmetry.

Now, for each hyperedge $e$ of $H$, let $\rho_e = \rho'_{e_1}$. I claim that $\rho_e$ satisfies Equations \ref{stat_dist_lem_eq_1} and \ref{stat_dist_lem_eq_2} for $H$. Equation \ref{stat_dist_lem_eq_2} is satisfied since
\begin{equation}
\omega(e) \cdot \rho_e = \omega(e) \cdot \rho'_{e_1} = \rho'_{e_1} + \cdots + \rho'_{e_{\omega(e)}} = \sum_{i=1}^{\omega(e)} \rho'_{e_i} \omega'(e_i),
\end{equation}
which implies
\begin{equation}
\sum_{e \in E} \rho_e \cdot \omega(e) = \sum_{e \in E} \sum_{i=1}^{\omega(e)} \rho'_{e_i} \omega'(e_i) = \sum_{e \in E'} \rho'_{e_i} \omega(e_i) = 1.
\end{equation}

To show Equation \ref{stat_dist_lem_eq_1} holds for $H$, note that
\begin{equation}
d(v)^{-1} \cdot \rho_f  \cdot \omega(f) \cdot \gamma_f(v) = \sum_{i=1}^{\omega(f)} \big(d(v)^{-1} \cdot \rho'_{f_i} \cdot \omega'(f_i) \cdot \gamma'_{f_i}(v)\big).
\end{equation}
Summing over both sides yields
\begin{equation}
\begin{split}
\sum_{v \in e} \sum_{f \in E(v)} d(v)^{-1} \cdot \rho_f  \cdot \omega(f) \cdot \gamma_f(v) &= \sum_{v \in e} \sum_{f \in E(v)} \sum_{i=1}^{\omega(f)} \big(d(v)^{-1} \cdot \rho'_{f_i} \cdot \omega'(f_i) \cdot \gamma'_{f_i}(v)\big) \\
&= \sum_{v \in e} \sum_{f \in M(v)} d(v)^{-1} \cdot \rho'_f \cdot \omega'(f) \cdot \gamma'_f(v) \\
&= \sum_{v \in e_1} \sum_{f \in M(v)} d(v)^{-1} \cdot \rho'_f \cdot \omega'(f) \cdot \gamma'_f(v) \\
&= \rho'_{e_1}, \text{since Equation \ref{stat_dist_lem_eq_1} holds for $H'$} \\
&= \rho_e.
\end{split}
\end{equation}

Thus, Equations \ref{stat_dist_lem_eq_1} and \ref{stat_dist_lem_eq_2} hold for $H$ when $H$ has rational hyperedge weights.

Finally, we consider the general case, where we assume nothing about the hyperedge weights besides that they are positive real numbers. By similar reasoning to our proof of the equal hyperedge weight case, we are done if we can find positive $\rho_e$ satisfying Equation \ref{stat_dist_lem_eq_1}.

We have
\begin{equation}
\begin{split}
\sum_{v \in e} \sum_{f \in E(v)} d(v)^{-1} \cdot \rho_f  \cdot \omega(f) \cdot \gamma_f(v) &= \sum_{v \in e} \sum_{f \in E(v)} d(v)^{-1} \cdot \omega(f) \cdot \rho_f  \cdot \gamma_f(v) \\
&= \sum_{f \in E} \sum_{v \in e\cap f} d(v)^{-1} \cdot \rho_f \cdot \omega(f) \cdot \gamma_f(v) \\
 &= \sum_{f \in E} \rho_f \cdot \left( \sum_{v \in e \cap f} d(v)^{-1} \cdot \omega(f) \cdot \gamma_f(v) \right).
\end{split}
\end{equation}

Let $A$ be a matrix of size $|E| \times |E|$ with entries 
\begin{equation}
\label{a_mat_eq}
A_{e,f} = \sum_{v \in e \cap f} d(v)^{-1} \cdot \omega(f) \cdot \gamma_f(y).
\end{equation}

Showing that there exist positive $\rho_e$ that satisfy Equation \ref{stat_dist_lem_eq_1} is equivalent to showing that $A$ has a positive eigenvector with eigenvalue $1$. By the Perron-Frobenius theorem, this equivalent to $A$ having spectral radius $1$.

For each hyperedge $e \in E$, let $q^e_1, q^e_2, \dots$ be a sequence of rational numbers that converges to $\omega(e)$, i.e. $\lim_{n \to\infty} q^e_n = \omega(e)$. Let $H_n$ be $H$ except we replace all hyperedge weights $\omega(e)$ with $q^e_n$. By the previous part of the proof, there exist positive constants $\rho^n(e)$ that satisfy Equation \ref{stat_dist_lem_eq_1} for $H_n$; equivalently, if we let $A_n$ be the matrix from Equation \ref{a_mat_eq} for hypergraph $H_n$, then $A_n$ has spectral radius $1$. 

Since $A_n$ has a continuous dependence on the hyperedge weights, and spectral radius is a continuous function, it follows that the spectral radius of $A$ is the limit of the spectral radius of $A_n$. Thus, the spectral radius of $A$ is $1$, and we are done.
\end{proof}

Theorem \ref{stat_dist_thm} is now a relatively straightforward corollary of Lemma \ref{stat_dist_lem}.

\begin{customthm}{\ref{stat_dist_thm}}
Let $H=(V, E, \omega, \gamma)$ be a hypergraph with edge-independent vertex weights. There exist positive constants $\rho_e$ such that the stationary distribution $\pi$ of a random walk on $H$ is
\begin{equation}
\label{stat_dist_hg}
\pi_v = \sum_{e \in E(v)} \omega(e) \cdot \big(\rho_e \gamma_e(v)\big).
\end{equation}

Moreover, $\rho_e$ can be computed in time $O\big(|E|^3 + |E|^2 \cdot |V|\big)$.
\end{customthm}

\begin{proof}[Proof of Theorem \ref{stat_dist_thm}]
Without loss of generality, assume $\delta(e) = \sum_{v \in e} \gamma_e(v) = 1$ for all hyperedges $e$, i.e. by scaling $\rho_e$ appropriately.

Let $\rho_e > 0$ be from Lemma \ref{stat_dist_lem}, and define
\begin{equation}
\pi_v = \sum_{e \in E(v)} \omega(e) \big( \rho_e \gamma_e(v) \big).
\end{equation}

I claim that $\pi_v$ is the stationary distribution for a random walk on $H$.

First, note that
\begin{equation}
\begin{split}
\sum_{v \in V} \pi_v &= \sum_{v \in V} \sum_{e \in E(v)} \omega(e) \big( \rho_e \gamma_e(v) \big) \\
&= \sum_{e \in E} \sum_{v \in e} \omega(e) \big( \rho_e \gamma_e(v) \big) \\
&= \sum_{e \in E} \rho_e \omega(e) \sum_{v \in e} \gamma_e(v) \\
&= \sum_{e \in E} \rho_e \omega(e) \\
&= 1, \text{ by Equation \ref{stat_dist_lem_eq_2}}
\end{split}
\end{equation}
so $\pi$ is indeed a probability distribution on $V$. Now, for any vertex $w \in V$, we have
\begin{equation}
\begin{split}
\sum_{v \in V} \pi_v p_{v, w} &= \sum_{v \in V} \pi_v \left( \sum_{e \in E(v)} \frac{\omega(e)}{d(v)} \gamma_e(w) \right) \\
&= \sum_{v \in V} \sum_{e \in E(v, w)} \pi_v \cdot \gamma_e(w) \cdot \omega(e) \cdot d(v)^{-1} \\
&= \sum_{e \in E(w)} \sum_{v \in e} \pi_v \cdot \gamma_e(w) \cdot \omega(e) \cdot d(v)^{-1} \\
&= \sum_{e \in E(w)} \omega(e) \cdot \gamma_e(w) \left( \sum_{v \in e} \frac{\pi_v}{d(v)} \right).
\end{split}
\end{equation}

If we simplify the inner sum, we get
\begin{equation}
\sum_{v \in e} \frac{\pi_v}{d(v)} = \sum_{v \in e} d(v)^{-1} \sum_{f \in E(v)} \rho_f \cdot \omega(f) \cdot \gamma_f(v) = \sum_{v \in e} \sum_{f \in E(v)} d(v)^{-1} \cdot \rho_f \cdot \omega(f) \cdot \gamma_f(v) = \rho_e.
\end{equation}

Plugging this back in, we get
\begin{equation}
\sum_{e \in E(w)} \omega(e) \cdot \gamma_e(w) \left( \sum_{v \in e} \frac{\pi_v}{d(v)} \right) = \sum_{e \in E(w)} \omega(e) \cdot \gamma_e(w) \cdot \rho_e = \pi_w.
\end{equation}

Thus, $\sum_{v \in V} \pi_v p_{v,w} = \pi_w$, so $\pi$ is a stationary distribution for $H$. 

Finally, note that computing $A$ (Equation \ref{a_mat_eq}) takes time $O(|E|^2 \cdot |V|)$ when $d(v)$ is precomputed, and solving $A\rho = \rho$ takes time $O(|E|^3)$, so the total runtime to compute $\rho_e$ is $O(|E|^3 + |E|^2 \cdot |V|)$.
\end{proof}

\section{Proof of Theorem \ref{mix_time_hg}}

For completeness, we include the definition of the Cheeger constant of a Markov chain \cite{montenegro}.

\begin{defn}
Let $M$ be an ergodic Markov chain with finite state space $V$, transition probabilities $p_{u,v}$, and stationary distribution $\pi$. The \emph{Cheeger constant} of $M$ is
\begin{equation}
\Phi = \min_{S \subset V, 0 < \pi(S) \leq 1/2} \frac{\sum_{x \in S, y \not\in S} \pi_x p_{x, y}}{\pi(S)},
\end{equation}
where $\pi(S) = \sum_{v \in S} \pi_v$.
\end{defn}

First, we prove the following lemma for the mixing time of any lazy Markov chain.

\begin{lem}
\label{lem_mix_time}
Let $M$ be a finite, irreducible Markov chain with states $S$ and transition probabilities $p_{x,y}$, satisfying $p_{x,x} \geq \delta$ for all $x \in S$. Let $\pi$ be the stationary distribution of $M$, and let $\pi_{min}$ be the smallest element of $\pi$. Then,
\begin{equation}
t_{mix}(\epsilon) \leq \left\lceil \frac{8\delta}{\Phi^2} \log\left( \frac{1}{2\epsilon\sqrt{\pi_{min}}} \right) \right\rceil
\end{equation}
\end{lem}

\begin{proof}[Proof of Lemma]

We use the notation of \citet{jerison}. Let $P^*$ be the time-reversal transition matrix of $P$. Note that $P^*P$ and $\frac{P+P^*}{2}$ are both reversible Markov chains. Let $\alpha$ be the square-root of the second-largest eigenvalue of $P^*P$, and let $b$ be the second-largest eigenvalue of $\frac{P+P^*}{2}$. By the Cheeger inequality, we have $1-b \geq \frac{\Phi^2}{2}$. Combining this with Lemma 1.21 of \citet{montenegro} yields
\begin{equation}
\label{jerison_eq_1}
\frac{\mathcal{E}_{\frac{P+P^*}{2}}(f,f)}{\Var_{\pi}(f)} \geq \frac{\Phi^2}{2},
\end{equation}
where $f : S \rightarrow \mathbb{R}$ is any function, $\mathcal{E}_{\frac{P+P^*}{2}}(f,f)$ is the Dirichlet form of the Markov chain $\frac{P+P^*}{2}$, and $\Var_{\pi}(f)$ is the variance of $f$ (see \citet{montenegro} for more details).

From \citet{jerison},
\begin{equation}
\label{jerison_eq_2}
\mathcal{E}_{P^*P}(f,f) \geq 2\delta \mathcal{E}_{\frac{P+P^*}{2}}(f,f).
\end{equation}
Combining Equations \ref{jerison_eq_1} and \ref{jerison_eq_2} yields
\begin{equation}
\frac{\mathcal{E}_{P^*P}(f,f)}{\Var_{\pi}(f)} \geq  \frac{\Phi^2}{4\delta}.
\end{equation}

Now, from Lemma 1.2 of \citet{montenegro}, $1-\alpha^2 \geq \frac{\mathcal{E}_{P^*P}(f,f)}{\Var_{\pi}(f)}$; plugging this into the above equation and rearranging yields $\alpha \leq \left( 1-\frac{\Phi^2}{4\delta} \right)^{1/2} \leq 1-\frac{\Phi^2}{8\delta}$. Plugging this into Equation 1.6 of \citet{jerison} yields

\begin{equation*}
t_{mix}(\epsilon) \leq \left\lceil \frac{1}{1-\alpha} \log\left( \frac{1}{2\epsilon\sqrt{\pi_{min}}} \right) \right\rceil \leq  \left\lceil \frac{8\delta}{\Phi^2}\log\left( \frac{1}{2\epsilon\sqrt{\pi_{min}}} \right) \right\rceil. \qedhere
\end{equation*}
\end{proof}

\begin{customthm}{\ref{mix_time_hg}}
Let $H=(V, E, \omega, \gamma)$ be a hypergraph with edge-dependent vertex weights. Without loss of generality, assume $\rho_e = 1$ (i.e. by multiplying the vertex weights in hyperedge $e$ by $\rho_e$).  Then,
\begin{equation}
\label{mix_hg_eq}
t_{mix}^H(\epsilon) \leq \left\lceil \frac{8\beta_1}{\Phi^2} \log\left( \frac{1}{2\epsilon \sqrt{d_{min} \beta_2}} \right) \right\rceil,
\end{equation}
where
\begin{itemize}
\item $\Phi$ is the Cheeger constant of a random walk on $H$ \cite{montenegro, jerison}
\item
$d_{min}$ is the minimum degree of a vertex in $H$, i.e. $d_{min} = \min_v d(v)$,
\item
$\beta_1 = \displaystyle\min_{e \in E, v \in e} \displaystyle \left( \frac{\gamma_e(v)}{\delta(e)}\right)$, and
\item
$\beta_2 = \displaystyle\min_{e \in E, v \in e} \big(\gamma_e(v)\big)$.
\end{itemize}
\end{customthm}

\begin{proof}[Proof of Theorem \ref{mix_time_hg}]
We have
\begin{equation}
p_{v,v} = \sum_{e \in E(v)} \frac{\omega(e)}{d(v)} \frac{\gamma_e(v)}{\delta(e)} \geq \beta_1 \sum_{e \in E(v)} \frac{\omega(e)}{d(v)} = \beta_1
\end{equation}
for all vertices $v$. Similarly,
\begin{equation}
\pi_v = \sum_{e \in E(v)} \omega(e) \gamma_e(v) \geq \beta_2 d(v).
\end{equation}
Applying Lemma \ref{lem_mix_time} to a random walk on $H$ yields the desired bound:
\begin{equation*}
t_{mix}(\epsilon) \leq \left\lceil \frac{8\delta}{\Phi^2} \log\left( \frac{1}{2\epsilon\sqrt{\pi_{min}}} \right) \right\rceil \leq \left\lceil \frac{8\beta_1}{\Phi^2} \log\left( \frac{1}{2\epsilon \sqrt{d_{min} \beta_2}} \right) \right\rceil \qedhere
\end{equation*}

\end{proof}

\section{Proof of Theorem \ref{hg_sg_eig}}

\begin{customthm}{\ref{hg_sg_eig}}
Let $H=(V,E,\omega,\gamma)$ be a hypergraph with edge-dependent vertex weights, with vertex weights normalized so that $\rho_e = 1$ for all hyperedges $e$. Let $G^H$ be the clique graph of $H$, with edge weights
\begin{equation}
\label{hg_sg_eig_weights}
w_{u, v} = \sum_{e \in E(u,v)} \frac{\omega(e) \gamma_e(u) \gamma_e(v)}{\delta(e)}.
\end{equation}

Let $L^H, L^G$ be the Laplacians of $H$ and $G^H$, respectively, and let $\lambda_1^H, \lambda_1^G$ be the second-smallest eigenvalues of $L^H, L^G$, respectively. Then
\begin{equation}
\frac{1}{c(H)} \lambda_1^H \leq \lambda_1^G \leq c(H) \lambda_1^H,
\end{equation}
where $\displaystyle c(H) = \max_{v \in V} \left( \frac{\max_{e \in E} \gamma_e(v)}{\min_{e \in E} \gamma_e(v)} \right)$.
\end{customthm}

\begin{proof}[Proof of Theorem \ref{hg_sg_eig}]
As shorthand, we write $G=G^H$. Let $p^H_{u,v}$ and $\pi^H_v$be the transition probabilities, stationary distribution of a random walk on $H$. Define $p^G_{u,v}$ and $\pi^G_v$ similarly for $G$. Furthermore, let $d^H(v)$ and $d^G(v)$ be the degree in $H$ and $G$ respectively. 

We will use Theorem 8 of \citet{chung_dg} to prove our theorem, which requires us to have lower and upper bounds on $\frac{\pi^G_v}{\pi^H_v}$ and $\frac{\pi^G_v p_{v,u}}{\pi^G_v p_{v,u}}$.

First, for an arbitrary vertex $v$, we have
\begin{equation}
\begin{split}
\pi^G_v \propto \sum_{u\in V} w_{u,v} &= \sum_{u\in V} \sum_{e \in E(u,v)} \frac{\omega(e) \gamma_e(u) \gamma_e(v)}{\delta(e)} \\
&= \sum_{e \in E(v)} \sum_{u \in e} \frac{\omega(e) \gamma_e(u) \gamma_e(v)}{\delta(e)} \\
&= \sum_{e \in E(v)} \omega(e) \gamma_e(v) \left( \frac{\sum_{u \in e} \gamma_e(u)}{\delta(e)} \right) \\
&= \sum_{e \in E(v)} \omega(e) \gamma_e(v) \\
&= \pi^H_v,
\end{split}
\end{equation}

so random walks on $G^H$ and $H$ have the same stationary distributions. Next, for any two vertices $u, v$, we have
\begin{equation}
\displaystyle\frac{\pi^G(v) p^G_{u,v}}{\pi^H(v) p^H_{u,v}} = \frac{p^G_{u,v}}{p^H_{u,v}} = \frac{ \displaystyle\frac{w_{u,v}}{d^G(u)} }{ \displaystyle\sum_{e \in E(u, v)} \frac{\omega(e)}{d^H(u)} \frac{\gamma_e(v)}{\delta(e)} } = \frac{ \displaystyle\sum_{e \in E(u, v)} \frac{\omega(e)\gamma_e(u) \gamma_e(v)}{\delta(e)} }{ \displaystyle\sum_{e \in E(u, v)} \frac{ \omega(e) \gamma_e(v) d^G(u) }{\delta(e) d^H(u)} }.
\end{equation}

The RHS is upper-bounded by the maximum ratio of each term in the sum, which is
\begin{equation}
\begin{split}
\max_{u,v} \frac{d^H(u) \gamma_e(u)}{d_G(u)} &= \max_{u,v} \frac{\left( \displaystyle\sum_{f \in E(u)} \omega(f) \right) \gamma_e(u)}{ \left( \displaystyle\sum_{f \in E(u)}  \omega(f) \gamma_f(u) \right) } \\[5pt]
&\leq \max_{u, v} \left( \frac{ \max_e \gamma_e(u) }{\min_f \gamma_f(u)} \right) \\[5pt]
&= \max_u \left( \frac{ \max_e \gamma_e(u) }{\min_e \gamma_f(u)} \right).
\end{split}
\end{equation}

Similarly, it is lower bounded by $\min_u \frac{\min_e \gamma_e(u)}{\max_e \gamma_e(u)}$. Applying Theorem 8 of \citet{chung_dg} gives the desired bound.
\end{proof}

\section{Rank Aggregation Experiments with Synthetic Data}


\textbf{Data:} We use a variant of the TrueSkill model to generate our data. We assume each player has an intrinsic ``skill" level (for simplicity, assume skill does not change over time), and a player's performance in match is proportional to their skill plus some added Gaussian noise. Such a model can represent many different kinds of games, including shooting games (e.g. Halo, scores represent kill/death ratios in a timed free-for-all match) and racing games (e.g. Mario Kart, scores are inversely proportional to the time a player takes to finish a course).

The players are $\{1, \dots, n\}$. Player $i$ has intrinsic skill $i$, so the true ranking of players, $\tau^*$, is 
\begin{equation*}
\text{player $1 < $ player $2 < \cdots $ player $n$}.
\end{equation*}
We create $k$ partial rankings, $\tau_1, \dots, \tau_k$, where each partial ranking $\tau_i$ corresponds to a noisy subsampling of $\tau^*$. More specifically, to create each partial ranking, we do the following. \vspace{-0.5em}
\begin{enumerate}
\item
Choose a subset of players $A \subset \{1, ..., n\}$, where player $i$ is included in $A$ with probability $p$.
\item
Choose a scale factor $c$ uniformly at random from $[1/3, 3]$.
\item
For each player $i \in A$, independently draw a score for player $i$ from a $N(0.2\cdot i, \sigma)$ distribution, and scale that score by $c$.
\item
Set $\tau_j$ to be a ranking of the players in $A$ according to their score.
\end{enumerate}

The tuneable parameters are: $n$, the number of players to be ranked; $\sigma$, the amount of noise in our partial rankings; $k$, the number of partial rankings; and $p$, which controls the size of each partial ranking.
We set the mean score for player $i$ to be $0.2\cdot i$, so that the the scale of the simulated scores is similar to the scores from the Halo dataset.

\textbf{Methods: } As with the real data, we create a Markov chain-based rank aggregation algorithm where the Markov chain is a random walk on a hypergraph $H=(V, E, \omega, \gamma)$. The vertices are $V=\{1, ..., n\}$, and the hyperedges $E$ correspond to the partial rankings  $\tau_1, \dots, \tau_k$.
 We set vertex weights $\gamma_{e_j}(v) = \exp[(\text{score of $v$ in partial ranking $\tau_j$})]$, and edge weights $\omega(e_j) = (\text{standard deviation of scores in $\tau_j$}) + 1$.

Our baselines are MC3 and a rank aggregation algorithm using the \emph{clique graph} $G^H$, both of which are described in the main text.

\textbf{Results: } We fix universe size $n=100$, and set $k$ to be the smallest number of hyperedges until all $n$ vertices are included at least once. We set $\sigma=1$ and $p=0.03, 0.05, 0.07$.

To assess performance, we measure the weighted Kendall $\tau$ correlation coefficient \cite{weightedkt} between the estimated ranking and the true ranking $\tau^*$. 
Our weighted hypergraph algorithm outperforms both MC3 and the clique graph algorithm in all cases (figure below), with the most significant gains when $p$ is small, i.e. when there is less information in each partial ranking. Moreover, the performance of the clique graph algorithm is much worse than both MC3 and the weighted hypergraph, which suggests that the clique graph is not a good approximation of $H$.

\begin{figure*}[ht]
\label{ra_res}
  \centering
  \subfigure{\includegraphics[width=.5\textwidth]{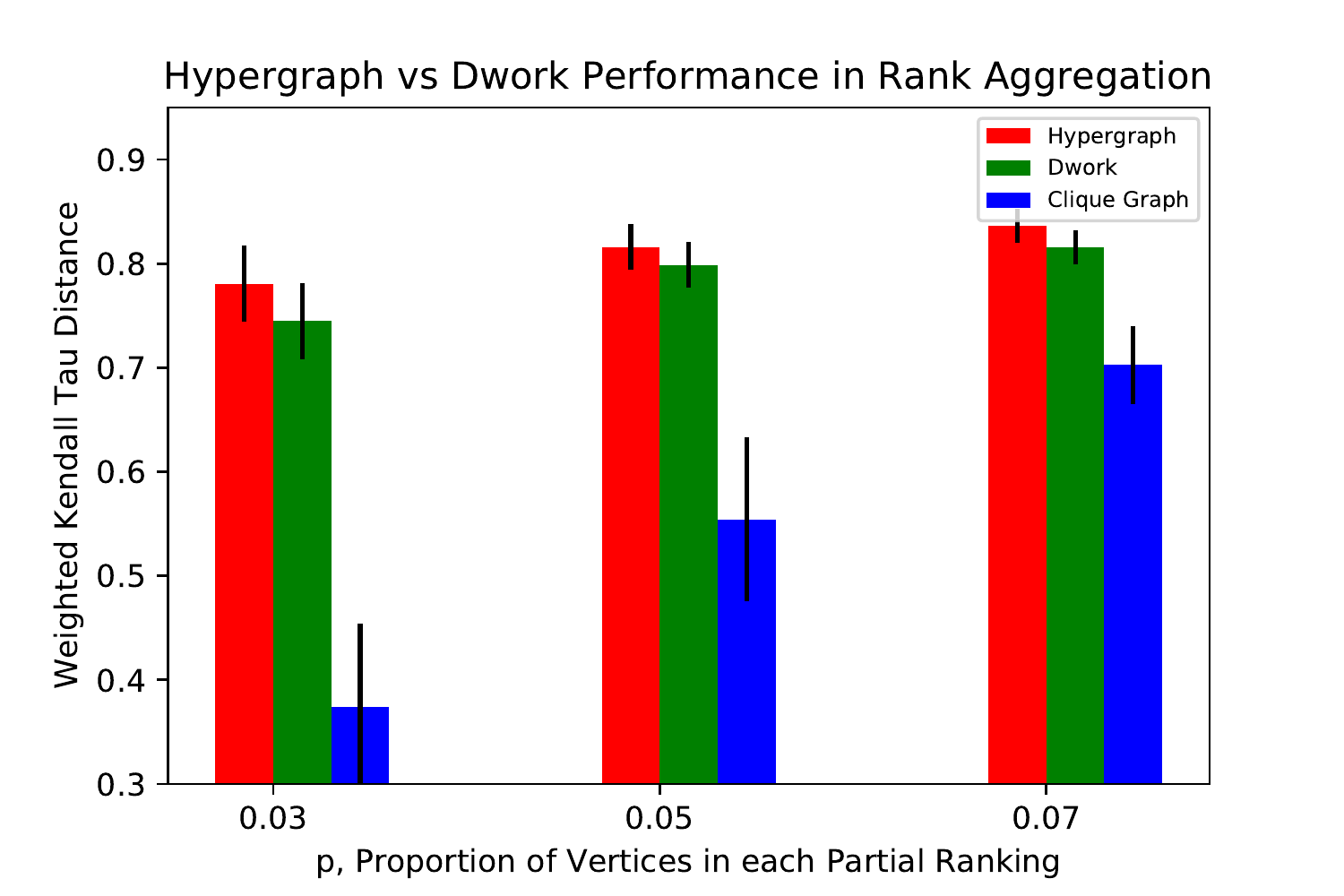}}
    \caption{Results of rank aggregation experiment using synthetic data.}
\end{figure*}

\end{document}